\documentclass[a4paper]{article}
\usepackage{amssymb}
\usepackage{graphicx}
\usepackage{geometry}
\usepackage{authblk}
\usepackage[parfill]{parskip}
\usepackage[round,authoryear,semicolon]{natbib}
\usepackage[bookmarks=false,colorlinks]{hyperref}
\usepackage{subfigure}
\usepackage{algorithm}
\usepackage{algorithmic}
\usepackage{tabularx}
\usepackage{enumitem}
\setlist{leftmargin=*,topsep=0em,itemsep=0em,parsep=0em}
\usepackage{cleveref}
\usepackage{mdl}
\usepackage{url}
\title{Differentially Private Empirical Risk Minimization with Input Perturbation}
\author[1]{Kazuto Fukuchi}
\author[2]{Quang Khai Tran~\thanks{This work was done when he was a master's student in the Dept. of Computer Science, Graduated School of SIE, University of Tsukuba}}
\author[1,3,4]{Jun Sakuma}
\affil[1]{Department of Computer Science, Graduate School of System and Information Engineering, University of Tsukuba}
\affil[2]{Intelligent Systems Laboratory, Secom Co., Ltd.}
\affil[3]{JST CREST}
\affil[4]{RIKEN Center for Advanced Intelligence Project}

\begin{document}
\maketitle

\begin{abstract}
We propose a novel framework for the differentially private ERM, input perturbation. Existing differentially private ERM implicitly assumed that the data contributors submit their private data to a database expecting that the database invokes a differentially private mechanism for publication of the learned model. In input perturbation, each data contributor independently randomizes her/his data by itself and submits the perturbed data to the database. We show that the input perturbation framework theoretically guarantees that the model learned with the randomized data eventually satisfies differential privacy with the prescribed privacy parameters. At the same time, input perturbation guarantees that local differential privacy is guaranteed to the server. We also show that the excess risk bound of the model learned with input perturbation is $O(1/n)$ under a certain condition, where $n$ is the sample size. This is the same as the excess risk bound of the state-of-the-art.
\end{abstract}

\section{Introduction}

In recent years, differential privacy has become widely recognized as a theoretical definition for output privacy~\citep{dwork2006calibrating}. Let us suppose a {\em database} collects private information from data contributors. {\em Analysts} can submit queries to learn knowledge from the database. Query-answering algorithms that satisfy differential privacy return responses such that the distribution of outputs does not change significantly and is independent of whether the database contains particular private information submitted by any single data contributor. Based on this idea, a great deal of effort has been devoted to guaranteeing differential privacy for various problems. For example, there are algorithms for privacy-preserving classification~\citep{jain2014near}, regression~\citep{lei2011differentially}, etc.

Differentially private empirical risk minimization~(ERM), or more generally, differentially private convex optimization, has attracted a great deal of research interest in machine learning, for example, \citep{chaudhuri2011differentially,kifer2012private,jain2014near,Bassily2014}. These works basically follow the standard setting of differentially private mechanisms; the database collects examples and builds a model with the collected examples so that the released model satisfies differential privacy.

Recently, the data collection process is also recognized as an important step in privacy preservation. With this motivation, a {\em local privacy} was introduced as a privacy notion in the data collection process~\citep{wainwright2012privacy,duchi2013local,kairouz2014extremal}. However, the existing methods of differentially private ERM are specifically derived for satisfying differential privacy of the released model, and thus there is no guarantee for the local privacy.

In this work, we aim to preserve the local privacy of the data and the differential privacy of the released model simultaneously in the setting of releasing the model constructed by ERM. The goal of this paper is to derive a differentially private mechanism with an utility guarantee, at the same time, the mechanism satisfies the local privacy in the data collection process.

 \begin{table}[t]
\label{tab:perturbation}
	\centering
  \newcolumntype{D}{>{\scriptsize\arraybackslash}p{.14\textwidth}}
  \newcolumntype{E}{>{\scriptsize\arraybackslash}p{.081\textwidth}}
  \renewcommand\tabularxcolumn[1]{>{\scriptsize}m{#1}}
	\caption{Comparison of differentially private ERM. All methods assume that $\ell_2$ norm of the parameters is bounded by $\eta$, the loss function is $\zeta$-Lipschitz continuous. $n$ and $d$ denote the number of examples and the dimension of the parameter, respectively.}
	\begin{tabularx}{\textwidth}{X|E|X|X|D}
		\hline
      Method & Pertur-bation & Privacy & Utility & Additional \newline requirements \\
    \hline\hline
      Objective~\citep{chaudhuri2011differentially,kifer2012private} & obj. func. & $(\epsilon,\delta)$-DP for model & $O\left(\frac{\eta\zeta\sqrt{d\log(1/\delta)}}{\epsilon n}\right)$ & $\lambda$-smooth \\
    \hline
      Gradient Descent~\citep{Bassily2014} & grad. & $(\epsilon,\delta)$-DP for model & $O\left(\frac{\eta\zeta\sqrt{d\log^2(n/\delta)}}{\epsilon n}\right)$ & \\
    \hline
      Input~(proposal) & example & $(\alpha\epsilon,\!\delta)$-DP for model \newline $(\beta\sqrt\epsilon,\!\delta)$-DLP for data \newline s.t. $O(\sqrt{\alpha n}) = \beta$ & $O\left(\frac{\eta\zeta\sqrt{d\log(1/\delta)}}{\epsilon \alpha n}\right)$ & $\lambda$-smooth \newline quadratic loss \\
    \hline
	\end{tabularx}
 \end{table}

{\bfseries Related Work.}
Chaudhuri et al.~\citet{chaudhuri2011differentially} formulated the problem of differentially private empirical risk minimization~(ERM) and presented two different approaches: output perturbation and objective perturbation.  Kifer et al.~\citet{kifer2012private} improved the utility of objective perturbation by adding an extra $\ell_2$ regularizer into the objective function. Moreover, they introduced a variant of objective perturbation that employs Gaussian distribution for the random linear term, which improves dimensional dependency from $O(d)$ to $O(\sqrt{d})$ whereas the satisfying privacy is relaxed from $(\epsilon,0)$-differential privacy to $(\epsilon,\delta)$-differential privacy~(Table \ref{tab:perturbation}, line 1). Objective perturbation is work well for smooth losses, whereas Bassily et al.~\cite{Bassily2014} proved that it is suboptimal for non-smooth losses. They developed the optimal algorithm of $(\epsilon,\delta)$-differentially private ERM, named {\em differentially private gradient descent}. It conducts the stochastic gradient decent where the gradient is perturbed by adding a Gaussian noise. They showed that the expected empirical excess risk of the differentially private gradient descent is optimal up to multiplicative factor of $\log n$ and $\log(1/\delta)$ even for non-smooth losses~(Table \ref{tab:perturbation}, line 2). They also provides the optimal mechanisms that satisfy $(\epsilon,0)$-differential privacy for strong and non-strong convex losses. Jain et al.~\citet{jain2014near} showed that for the specific applications, the dimensional dependency of the excess risk can be improved from polynomic to constant or logarithmic. These studies assume that the database collects raw data from the data contributors, and so no attention has been paid to the data collection phase.

Recently, a new privacy notion referred to as {\em local privacy}~\citep{wainwright2012privacy,duchi2013local,kairouz2014extremal} has been presented. In these studies, data are drawn from a distribution by each contributor independently and communicated to the data collector via a noisy channel; local privacy is a privacy notion that ensures that data cannot be accurately estimated from individual privatized data. \citet{duchi2013local} has introduced a private convex optimization mechanism that satisfies the local privacy. Their method has guarantee of differential privacy for the model, whereas its privacy level is same as the differential local privacy.

{\bfseries Our Contribution.}
In this study, we propose a novel framework for the differentially private ERM, input perturbation (Table \ref{tab:perturbation}, line 3). In contrast to the existing methods, input perturbation allows data contributors to take part in the process of privacy preservation of model learning. The mechanism of input perturbation is quite simple: each data contributor independently randomizes her/his data with a Gaussian distribution, in which the noise variance is determined by a function of privacy parameters $(\epsilon, \delta)$, sample size $n$, and some constants related to the loss function.

In this paper, we prove that models learned with randomized examples following our input perturbation scheme are guaranteed to satisfy $(\alpha\epsilon,\delta)$-differential privacy under some conditions, especially, $(\epsilon,\delta)$-differential privacy if $\alpha=1$~(Table \ref{tab:perturbation}, line 3, column 3). The guarantee of differential privacy is proved using the fact that the difference between the objective function of input perturbation and that of objective perturbation is probabilistically bounded. To achieve this approximation with randomization by independent data contributors, input perturbation requires that the loss function be quadratic with respect to the model parameter, $\vec{w}$~(Table \ref{tab:perturbation}, line 3, column 5).

From the perspective of data contributors, data collection with input perturbation satisfies the local privacy with the privacy parameter $(\beta\sqrt\epsilon,\delta)$ where $\beta = O(\sqrt{\alpha n})$~(Table \ref{tab:perturbation}, line 3, column 3). In the input perturbation framework, not only differential privacy of the learned models, but also privacy protection of data against the database is attained. From this perspective, we theoretically and empirically investigate the influence of input perturbation on the excess risk.

We compared the utility analysis of input perturbation with those of the output and objective perturbation methods in terms of the expectation of the excess empirical risk. We show that the excess risk of the model learned with input perturbation is $O(1/\alpha n)$~(Table \ref{tab:perturbation}, line 3, column 4). If $\alpha=1$, the utility and the privacy guarantee of the model are equivalent to that of objective perturbation.

All proofs defer to the full version of this paper due to space limitation.

\section{Problem Definition and Preliminary}\label{sec:pre-pro}
Let $\family{Z} = \family{X}\times\family{Y}$ be the domain of examples. The objective of supervised prediction is to learn a parameter $\vec{w}$ on a closed convex domain $\family{W} \subseteq \RealSet^d$ from a collection of given examples $D = \{(\vec{x}_i,y_i)\}_{i=1}^n$, where $\vec{w}$ parametrizes a predictor that outputs $y \in \family{Y}$ from $x \in \family{X}$. Let $\ell:\family{W}\times\family{Z}\to\RealSet$ be a loss function. Learning algorithms following the empirical risk minimization principle choose the model that minimizes the empirical risk:
\begin{align}
	J(\vec{w};D) = \frac{1}{n}\sum_{i=1}^{n}\ell(\vec{w},( \vec{x}_i,y_i)) + \frac{1}{n}\Omega(\vec{w}),
	\label{eq:objective_function}
\end{align}
where $\Omega(\vec{w})$ is a convex regularizer. We suppose that the following assumptions hold throughout this paper: 1) $\family{W}$ is bounded, i.e., there is $\eta$ s.t. $\|\vec{w}\|_2 \le \eta$ for all $\vec{w} \in \family{W}$, 2) $\ell$ is doubly continuously differentiable w.r.t. $\vec{w}$, 3) $\ell$ is $\zeta$-Lipschitz, i.e., $\|\nabla\ell(\vec{w},(x,y))\|_2 \le \zeta$ for any $\vec{w} \in \family{W}$ and $(x,y) \in \family{Z}$, and 4) $\ell$ is $\lambda$-smooth, i.e., $\|\nabla^2\ell(\vec{w},(x,y))\|_2 \le \lambda$ for any $\vec{w} \in \family{W}$ and $(x,y) \in \family{Z}$ where $\|\cdot\|$ is the $\ell_2$ matrix norm.

Three stakeholders appear in the problem we consider: {\em data contributors}, {\em database}, and {\em model user}. Each data contributor owns a single example $(\vec{x}_i,y_i)$. The goal is that the model user obtains the model $\vec{w}$ learned by ERM, at the same time, privacy of the data contributors is ensured against the database and the model user. Let us consider the following process of data collection and model learning.
\begin{enumerate}
\item All the stakeholders reach an agreement on the privacy parameters $(\epsilon, \delta)$ before data collection
\item Each data contributor independently perturbs its own example and sends it to the database
\item The database conducts model learning at the request of the model user with the collected perturbed examples and publishes the model
\end{enumerate}
Note that once a data contributor sends her perturbed example to the database, she can no longer interact with the database. This setting is suitable for real use, for example, if the data contributors sends their own data to the database via their smartphones, the database is difficult to always interact with the data contributors due to instability of internet connection. In this process, the privacy concerns arise at two occasions; when the data contributors release their own data to the database~(data privacy), and when the database publishes the learned model to the model user~(model privacy).

\noindent{\bfseries Model privacy.}
\sloppy The model privacy is preserved by guaranteeing the $(\epsilon,\delta)$-differential privacy. It is a privacy definition of a randomization mechanism ${\cal M}$ which is a stochastic mapping from a set of examples $D$ to an output on an arbitrary domain $\family{O}$. Given two databases $D$ and $D'$, we say $D$ and $D'$ are neighbor databases, or $D \sim D'$, if two databases differ in at most one element. Then, differential privacy is defined as follows:
\begin{definition}
[$(\epsilon, \delta)$-differential privacy~\citep{dwork2006our}] A randomization mechanism ${\cal M}$ is $(\epsilon, \delta)$-differential privacy, if, for all pairs $(D,D')$ s.t. $D \sim D'$ and for any subset of ranges $S \subseteq {\cal O}$,
\begin{align}
\mathrm{Pr}[{\cal M}(D) \in S] \le \exp(\epsilon) \mathrm{Pr}[{\cal M}(D') \in S] +\delta \label{eq:dp}.
\end{align}
\end{definition}

\noindent{\bfseries Data privacy.}
For the definition of the data privacy, we introduce the differential local privacy~\citep{wainwright2012privacy,duchi2013local,kairouz2014extremal}. Because of the data collection and model learning process, the non-interactive case of the local privacy should be considered, where in this case, individuals release his/her private data without seeing the other individuals' private data. Under the non-interactive setting, the differential local privacy is defined as follows.
\begin{definition}
[$(\epsilon,\delta)$-differential local privacy~\citep{wainwright2012privacy,evfimievski2003limiting,kasiviswanathan2011can}] A randomization mechanism ${\cal M}$ is $(\epsilon, \delta)$-differentially locally private, if, for all pairs $(z,z')$ s.t. $z \ne z'$ and for any subset of ranges $S \subseteq {\cal O}$,
\begin{align}
\mathrm{Pr}[{\cal M}(z) \in S] \le \exp(\epsilon) \mathrm{Pr}[{\cal M}(z') \in S] +\delta \label{eq:localdp}.
\end{align}
\end{definition}

\noindent{\bfseries Utility.}
\sloppy To assess utility, we use the {\em empirical excess risk}. Let $\hat{\vec{w}} = \argmin_{\vec{w} \in \family{W}}J(\vec{w};D)$. Given a randomization mechanism ${\cal M}$ that (randomly) outputs $\vec{w}$ over $\family{W}$, the empirical excess risk of ${\cal M}$ is defined as $J({\cal M}(D);D) - J(\hat{\vec{w}};D)$.

\section{Input Perturbation}\label{sec:cs_dp_model_learning}

In this section, we introduce a novel framework for differentially private ERM. The objective of the input perturbation framework is three-fold:
\begin{itemize}
\item (data privacy) The released data from the data contributors to the database satisfies $(O(\sqrt{n\epsilon}), \delta)$-differentially locally private,
\item (model privacy) The model resulted from the process eventually meets $(\epsilon, \delta)$-differentially private,
\item (utility) The expectation of the excess empirical risk of the resulting models is $O(1/n)$, which is equivalent to that obtained with non-privacy-preserving model learning.
\end{itemize}
Furthermore, we show that by adjusting the noise variance that the input perturbation injects, the input perturbation satisfies $(\alpha\epsilon,\delta)$-differential privacy and $(\beta\epsilon,\delta)$-differential local privacy with the $O(1/\alpha n)$ excess empirical risk where $\beta = O(\sqrt{\alpha n})$.

\subsection{Loss Function for Input Perturbation}\label{sec:qua_loss_func}
The strategy of input perturbation is to minimize a function that is close to the objective function of the objective perturbation method. The requirements on the loss and objective function thus basically follow the objective perturbation method with the Gaussian noise~\citep{kifer2012private}. Input perturbation allows any (possibly non-differential) convex regularizer as supported by objective perturbation. However, for simplicity, we consider the non-regularized case where $\Omega(\vec{w}) = 0$.

In addition to the requirements from the objective perturbation, input perturbation requires a restriction; the loss function is quadratic in $\vec{w}$. Let $\vec{q}(\vec{x}_i,y_i)$ and $\vec{p}(\vec{x}_i,y_i)$ be $d$ dimensional vectors and $s(\vec{x}_i,y_i)$ be a scalar. Then, our quadratic loss function has a form:
\begin{align}
\ell(\vec{w},(\vec{x}, y))= \frac{1}{2} \vec{w}^T\vec{q}(\vec{x}, y)\vec{q}(\vec{x},y)^T\vec{w} - \vec{p}(\vec{x},y)^T\vec{w} + s(\vec{x},y).  \nonumber
\end{align}

\subsection{Input Perturbation Method}\label{sec:input_algorithm}

In this subsection, we introduce the input perturbation method. Algorithm \ref{al:safe_sharing} describes the detail of input perturbation; Algorithm \ref{al:proposed} describes model learning with examples
randomized with input perturbation. In Algorithm \ref{al:safe_sharing}, each data contributor transforms owing example $(\vec{x}_i, y_i)$ into $(\vec{q}_i,\vec{p}_i)$, where $\vec{q}_i =\vec{q}(\vec{x}_i, y_i), \vec{p}_i = p(\vec{x}_i,y_i)$. Then, she adds perturbation to $(\vec{q}_i,\vec{p}_i)$ in Step 3. We denote the example after perturbation by $(\tilde{\vec{q}}_i, \tilde{\vec{p}}_i)$, which is submitted to the database independently by each data contributors.

In Algorithm \ref{al:proposed}, the database collects the perturbed examples $\tilde{D}=\{\tilde{\vec{q}}_i, \tilde{\vec{p}}_i\}_{i=1}^n$ from the $n$ data contributors. Then, the database learns a model with these randomized examples by minimizing
\begin{align}
 J^{in}(\vec{w};\tilde{D}) = \frac{1}{n}\sum_{i=1}^n\left(\frac{1}{2}\vec{w}^T\tilde{\vec{q}}_i\tilde{\vec{q}}_i^T\vec{w} - \tilde{\vec{p}}_i^T\vec{w} + s_i\right) + \frac{\Delta_{in}}{2n}\norm{\vec{w}}_2^2. \label{eq:input_function}
\end{align}

In the following subsections, we show the privacy guarantee of the input perturbation in the sense of the differential local privacy and the differential privacy. The utility analysis of models obtained following the input perturbation framework is also shown.

\begin{algorithm}[t]
\caption{Input Perturbation}
\textbf{Public Input: } $\epsilon, \delta, d, n, \eta, \zeta$ and $\lambda$ \\
\textbf{Input of data contributor $i$: }$\vec{x}_i, y_i$ \\
\mbox{\textbf{Output of data contributor $i$: }$\tilde{\vec{q}}_i,\tilde{\vec{p}}_i$}\begin{algorithmic}[1]
\STATE $\gamma,\delta' \gets \frac{\delta}{2}, a = \sqrt{\frac{\log(2/\gamma)}{n}}$,
$\sigma_b^2 \gets \frac{\zeta^2( 8\log2/\delta' + 4\epsilon )}{\epsilon^2}$,
$\sigma_u^2 > \left(\frac{\sqrt{2d}a\lambda + \sqrt{2da^2\lambda^2 + \frac{2\lambda}{\epsilon}(1-2a)}}{(1-2a)}\right)^2$
\STATE Sampling of noise vectors: $\vec{r}_i \sim \mathcal{N}(0, \frac{\sigma_b^2}{n}\mathbb{I})$,
$\vec{u}_i \sim \mathcal{N}(0, \frac{\sigma_u^2}{n}\mathbb{I})$
\STATE $\tilde{\vec{q}}_i \gets \vec{q}_i + \vec{u}_i ,
\tilde{\vec{p}}_i \gets \vec{p}_i - \vec{r}_i$ where $\vec{q}_i = q(\vec{x}_i,y_i)$ and $\vec{p}_i = p(\vec{x}_i,y_i)$

\STATE Submission:
Send $\tilde{\vec{q}}_i, \tilde{\vec{p}}_i$ to the database
\end{algorithmic}
\label{al:safe_sharing}
\end{algorithm}
\begin{algorithm}[t]
	\caption{Model Learning on Input Perturbation}
	\begin{algorithmic}[1]
\REQUIRE $\epsilon, \delta, d, n, \eta, \zeta$ and $\lambda$
\STATE All stakeholders agree with $(\epsilon, \delta)$ and share parameters $d, n, \eta, \zeta$ and $\lambda$.
\STATE The database collects $(\tilde{\vec{q}}_i, \tilde{\vec{p}}_i)$ from the data contributors with Algorithm \ref{al:safe_sharing}.
\STATE The database learns $\vec{w}^{in} = \argmin_{\vec{w}\in\family{W}}J^{in}(\vec{w};\tilde{D})$ with $\Delta_{in} = \Delta-\frac{2\lambda}{\epsilon}$.
		\STATE Return $\vec{w}^{in}$.
	\end{algorithmic}
\label{al:proposed}
\end{algorithm}

\subsection{Privacy of Input Perturbation}\label{sec:input_privacy}

In this subsection, we analyze the privacy of the input perturbation in the sense of the data privacy and the model privacy.

\noindent{\bfseries Data privacy of input perturbation.}
In Algorithm \ref{al:safe_sharing}, each data contributor of the input perturbation adds a Gaussian noise into the released data. Adding a Gaussian noise into the released data satisfies $(\epsilon,\delta)$-differential local privacy as well as the Gaussian mechanism~\citep{dwork2014algorithmic}. As a result, we get the following corollary that shows the level of the differential local privacy of Algorithm \ref{al:safe_sharing}.
\begin{corollary}\label{cor:data-priv}
 \sloppy Suppose that $q$ and $p$ in Algorithm \ref{al:safe_sharing} are in the bounded domain with the size parameter $B$. Then, Algorithm \ref{al:safe_sharing} satisfies $(2c\sqrt{n}(\lambda/\sigma_u + \zeta/\sigma_b), 2\delta)$-differential local privacy, where $c > \sqrt{2\ln(1.25/\delta)}$.
\end{corollary}
Since we have $\lambda/\sigma_u + \zeta/\sigma_b \to (\sqrt{\frac{\lambda}{2}} + \sqrt{\frac{\epsilon}{8\log (2/\delta') + 4\epsilon}})\sqrt\epsilon$ as $n \to \infty$, Algorithm \ref{al:safe_sharing} is $(O(\sqrt{n\epsilon}), \delta)$-differentially locally private.

\noindent{\bfseries Model privacy of input perturbation.}
The following theorem states the guarantee of differential privacy of models that the database learns from examples randomized by the input perturbation scheme.
\begin{theorem}\label{th:proposed_dp}
Let $\tilde{D}$ be examples perturbed by Algorithm \ref{al:safe_sharing} with privacy parameters $\epsilon$ and $\delta$. Then, if $\Delta>\frac{2\lambda}{\epsilon}$ and $\gamma = \frac{\delta}{2}$, the output of Algorithm \ref{al:proposed} satisfies $(\epsilon, \delta)$-differential privacy.
\end{theorem}
The main idea of the proof is that the objective function of the input perturbation scheme holds the same linear perturbation term as that of objective perturbation. The objective function of input perturbation in Eq. \ref{eq:input_function} is rearranged as:
\begin{align}
 J^{in}(\vec{w};\tilde{D}) = \sum_i \ell (\vec{w},(\vec{x}_i,y_i)) + \vec{b}^T \vec{w} + \frac{\Delta_0 + \Delta - \frac{2\lambda}{\epsilon}}{2n}\vec{w}^T\vec{w}. \label{eq:input_perturbation_function_edited}
\end{align}
where $\frac{\Delta_0}{2n}\vec{w}^T \vec{w}= \frac{1}{2n}\vec{w}^T( \vec{U}^T\vec{U}+\vec{U}^T\vec{Q}+\vec{Q}^T\vec{U} )\vec{w}$ and $\vec{U} = [\vec{u}_1,\cdots,\vec{u}_n]^T$. The derivation can be found in the proof of Theorem \ref{th:proposed_dp}. In the linear term, $\vec{b}$ forms a random vector generated from $\mathcal{N}(0,\frac{\zeta^2(8\log(2/\delta')+4\epsilon)}{\epsilon^2}\mathbb{I})$, which is exactly the same as the random linear regularization term introduced in the objective perturbation method. By noting that $\lim_{n\rightarrow \infty}\vec{U}^T\vec{U} = \frac{2\lambda}{\epsilon}\mathbb{I}$ and $\lim_{n\rightarrow \infty}\vec{U}^T\vec{Q}  = \lim_{n\rightarrow \infty}\vec{Q}^T\vec{U}  = \vec{O}$, the objective function of input perturbation is equivalent to that of objective perturbation with an infinitely large number of samples.

For guarantee of differential privacy with a finite number of samples $n$, we use the following probabilistic bound of $\Delta_0$.
\begin{lemma}\label{lem:lem1}
 \sloppy Let $\vec{U} = [\vec{u}_1,\cdots,\vec{u}_n]^T$, where $\vec{u}_i \sim \mathcal{N}(0, \frac{\sigma_u^2}{n}\mathbb{I}_d)$. Let $\Delta_0 \vec{w}^T\vec{w} = \vec{w}^T(\vec{U}^T\vec{U}+\vec{U}^T\vec{Q}+\vec{Q}^T\vec{U})\vec{w}$. Then, for any $\gamma > 0$, with probability at least $1-\gamma$, we get the following bound:
	\begin{align}
		\varkappa(n, \gamma) \leq \Delta_0 \leq \kappa(n,\gamma), \nonumber
	\end{align}
	where
	\begin{align}
		\kappa(n,\gamma) =& \sigma_u^2\left(1 +
	 2\sqrt{\frac{\log(4/\gamma)}{n}} + 2
	 \frac{\log(4/\gamma)}{n}\right) + 2\sqrt{2d}\lambda\sigma_u\sqrt{\frac{\log(2/\gamma)}{n}} \nonumber \\
		\varkappa(n,\gamma) =& \sigma_u^2\left(1 - 2\sqrt{\frac{\log(4/\gamma)}{n}}\right)  - 2\sqrt{2d}\sigma_u\lambda\sqrt{\frac{\log(2/\gamma)}{n}}. \nonumber
  \end{align}
\end{lemma}
The proof can be found in \cref{lem:lem1}. This bound shows how $\Delta_0$ generated with $n$ samples is distant from $\sigma_u^2$. Setting $\sigma_u^2$ as in Algorithm \ref{al:safe_sharing}, we can get $\Delta_0 \ge \frac{2\lambda}{\epsilon}$ w.p. $1-\gamma$. Thus, the output of input perturbation guarantees $(\epsilon, \delta)$-differential privacy w.p. $1-\gamma$. The proof of Theorem \ref{th:proposed_dp} is obtained by incorporating the probabilistic bound on $\Delta_0$ into the privacy proof of \citep{kifer2012private}.

\subsection{Utility Analysis}\label{sec:input_general_bound}

The following theorem shows the excess empirical error bound of the model learned by input perturbation:
\begin{lemma}\label{lm:empirical_bound}
	Let $\vec{w}^{in}$ be the output of \cref{al:proposed}. If $\Delta>\frac{2\lambda}{\epsilon}$ and examples are randomized by \cref{al:safe_sharing}, w.p. at least $1-\gamma-\beta$ the bound of the excess empirical risk is
\begin{multline}
 J(\vec{w}^{in};D) - J(\hat{\vec{w}};D)
  \le \frac{4d\zeta^2 (8\log \frac{4}{\delta} + 4\epsilon) \log\frac{1}{\beta}}{n\epsilon^2 \Delta} + \frac{\Delta }{2n} \| \hat{\vec{w}}\|_2^2 + \frac{\sigma_u^2 - \frac{2\lambda}{\epsilon}}{2n} \| \hat{\vec{w}}\|_2^2 \nonumber \\
   + \frac{\sigma_u^2\sqrt{\log\frac{4}{\gamma}} + \sigma_u^2\frac{\log\frac{4}{\gamma}}{\sqrt{n}} + \sigma_u\lambda\sqrt{2d\log\frac{2}{\gamma}}}{n\sqrt{n}} \| \hat{\vec{w}}\|_2^2 \nonumber
\end{multline}
\end{lemma}
In the right side of the bound, the first two terms of $O(1/n)$ are the same as the excess empirical risk of objective perturbation~\citep{kifer2012private}. The third term of $O(1/n)$ and the last term of $O(1/n^{3/2})$ are introduced by input perturbation. The same holds with expectation of the excess risk, as stated in the following theorem.
\begin{theorem}\label{th:excess_bound}
 Let $\vec{w}^{in}$ be the output of Algorithm \ref{al:proposed}. If $\Delta>\frac{2\lambda}{\epsilon}$, $n\ge 16\log\frac{8}{\delta}$, and examples are randomized by Algorithm \ref{al:safe_sharing}, expectation of the excess empirical risk $E \left[ J(\vec{w}^{in}; D ) - J(\hat{\vec{w}}; D ) \right] = O\left(\frac{\zeta\norm{\hat{\vec{w}}}_2\sqrt{d\log(1/\delta)}}{\epsilon n}\right)$ by setting $\Delta = \Theta\left(\frac{\sqrt{\zeta^2d\log(1/\delta)}}{\epsilon\norm{\hat{\vec{w}}}_2}\right)$ and $\sigma_u$ as the lowest value specified in \cref{al:safe_sharing}.
\end{theorem}

\subsection{Balancing Local Privacy and Utility}
The privacy parameters of the differential local privacy that satisfy the input perturbation are $(O(\sqrt{n\epsilon}),\delta)$. Unfortunately, the privacy level of the input perturbation becomes weaker as the sample size $n$ increases. However, the input perturbation can satisfy stronger differential local privacy by adjusting $\epsilon$. If the data contributors require stronger local privacy, we set $\epsilon \gets \alpha\epsilon$ for small $\alpha \in (0,1)$, which enables the input perturbation to satisfy $(O(\sqrt{\alpha n \epsilon}),\delta)$-differential local privacy. Such setting of $\epsilon$ results in a higher privacy guarantee of the published model as $(\alpha\epsilon,\delta)$-differentially private, and a lower utility as $O(1/\alpha n)$.

\section{Experiments}\label{sec:experiment}

In this section, we examine the performance of input perturbation by experimentation. As predicted by Theorem \ref{th:excess_bound}, under the same privacy guarantee and the optimal setting of $\sigma_u^2$ in input perturbation, the expectation of the excess empirical risk of the models learned with the input perturbation and the objective perturbation is the same as $O(\frac{\zeta\norm{\hat{\vec{w}}}_2\sqrt{d\log(1/\delta)}}{\epsilon n})$. We experimentally evaluate the difference between the input perturbation and the objective perturbation with real datasets while changing the size of training data and privacy parameters. We compared the performance of the input perturbation method~(Input) against two methods, namely, output perturbation with Laplace mechanism~(Output),~\citep{chaudhuri2011differentially}, and objective perturbation with the Gaussian mechanism (Obj-Gauss)~\citep{kifer2012private}.
We evaluated all approaches to learn the linear regression model and the logistic regression model. For the performance measure, the root mean squared error (RMSE) was used for the linear regression model and the prediction accuracy was used for the logistic regression model.
For regularization parameter tuning, with each method we found the best parameter for the largest size of training dataset, then used it for other sizes of the training dataset.

In each experiment, we randomly divided the examples into a training dataset and a test dataset with the ratio $4:1$; we trained the model with the training dataset and evaluated the performance measure with the test dataset. The average results over 100 trials were reported. We used IBM ILOG CPLEX Optimizer to optimize the objective function.

\begin{figure}[t]
  \centering
	\subfigure[$\epsilon=0.1$]{\includegraphics[width=0.45\textwidth]{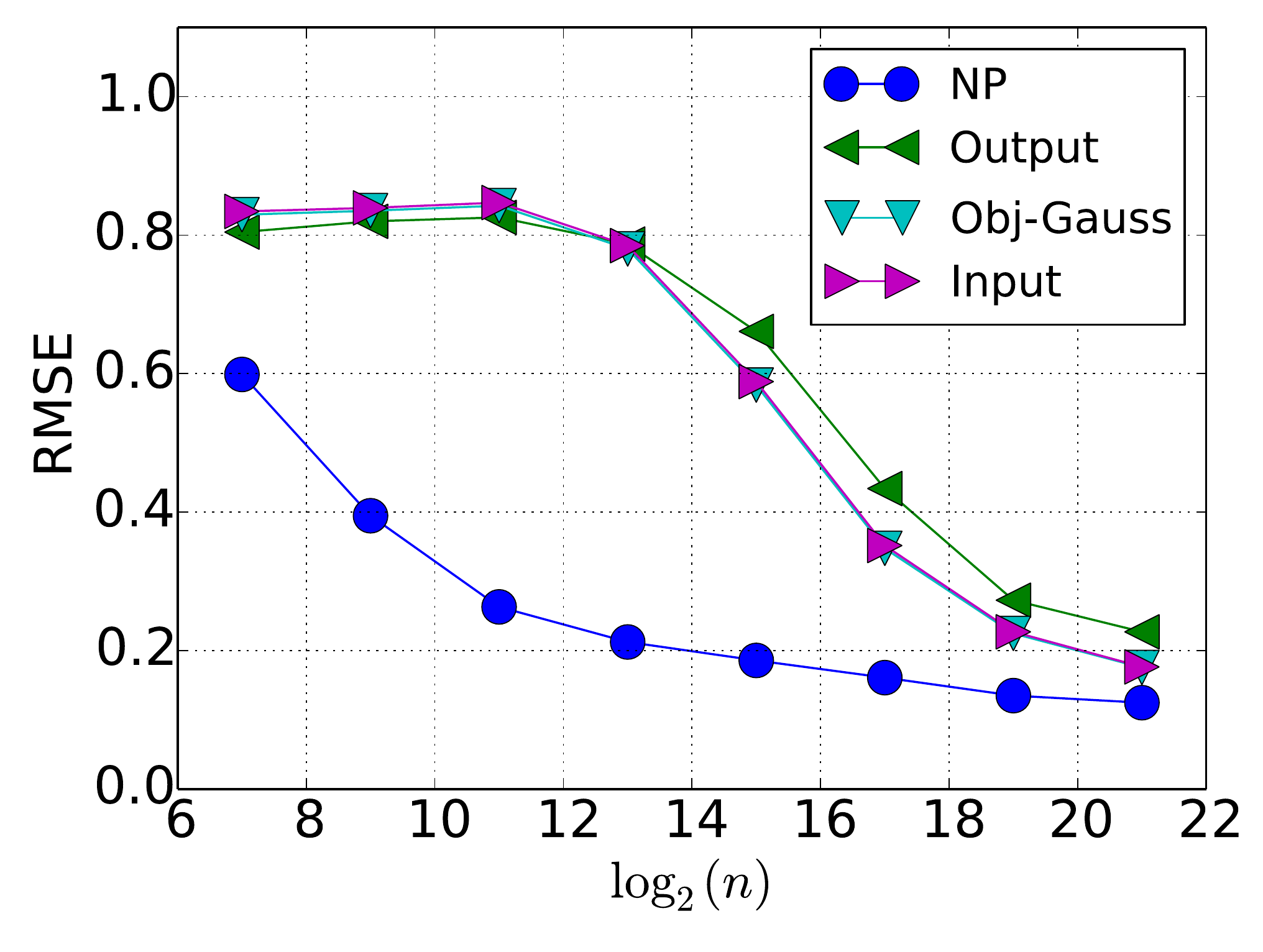}}
	\subfigure[$\epsilon=1.0$]{\includegraphics[width=0.45\textwidth]{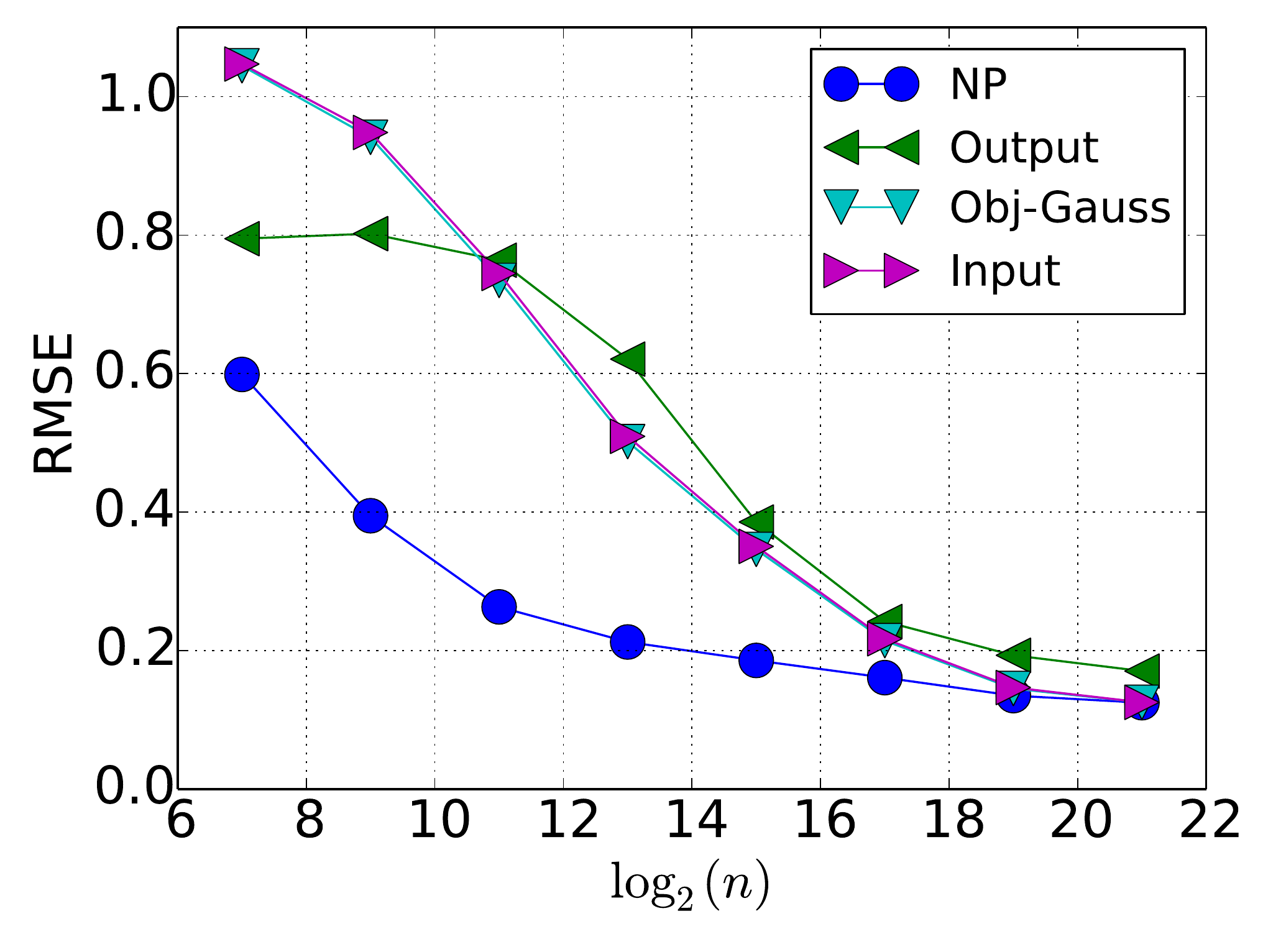}}
\caption{RMSEs of differentially private linear regression. The results were averaged over 100 trials while changing the example size $n$. We compared input perturbation~(Input), output perturbation with Laplace mechanism~(Output), and objective perturbation with Gaussian mechanism~(Obj-Gauss), and non-private linear regression as the baseline~(NP). }
		\label{fig:linear_result}
	\subfigure[$\epsilon=0.1$]{\includegraphics[width=0.45\textwidth]{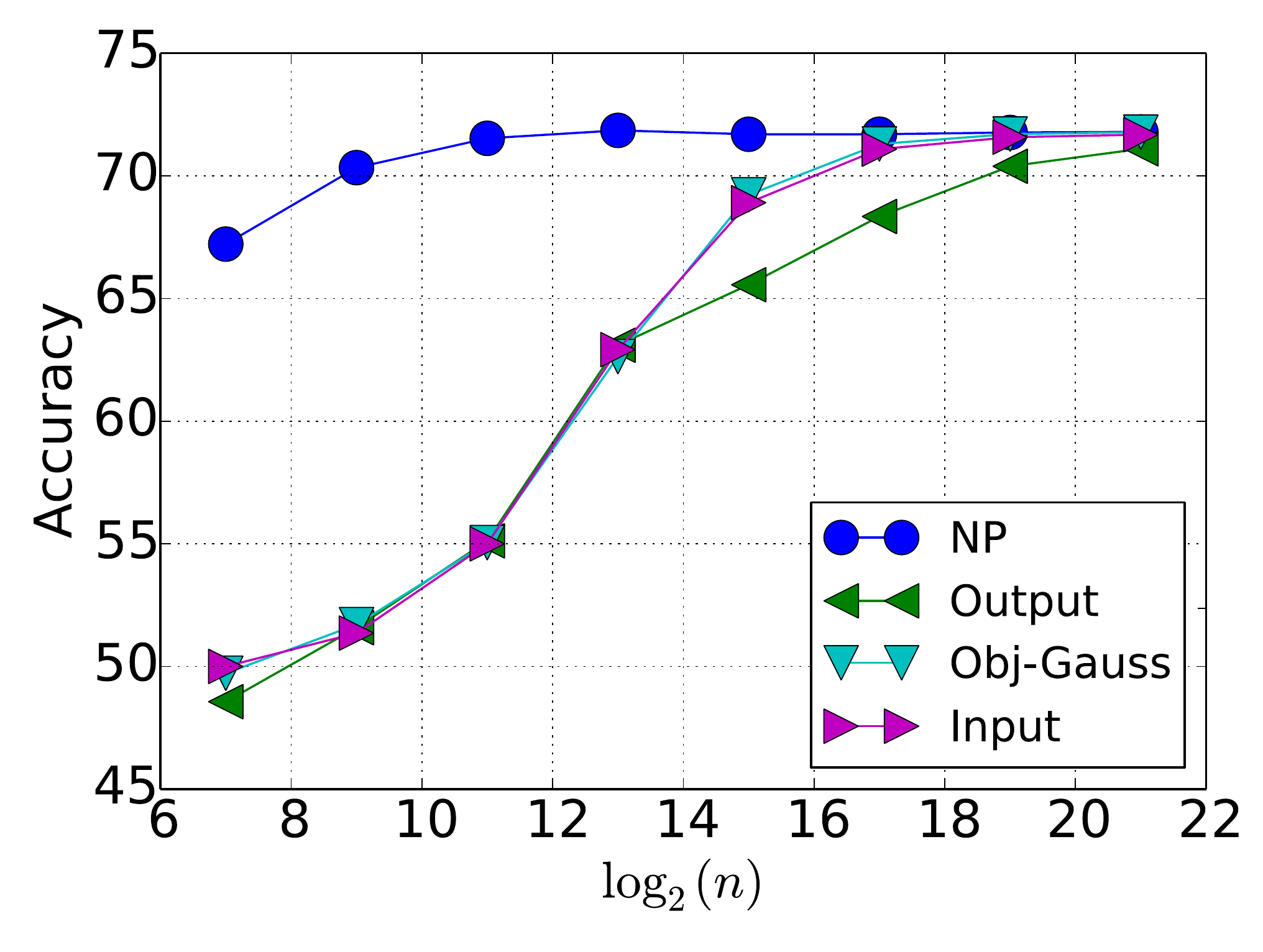}}
	\subfigure[$\epsilon=1.0$]{\includegraphics[width=0.45\textwidth]{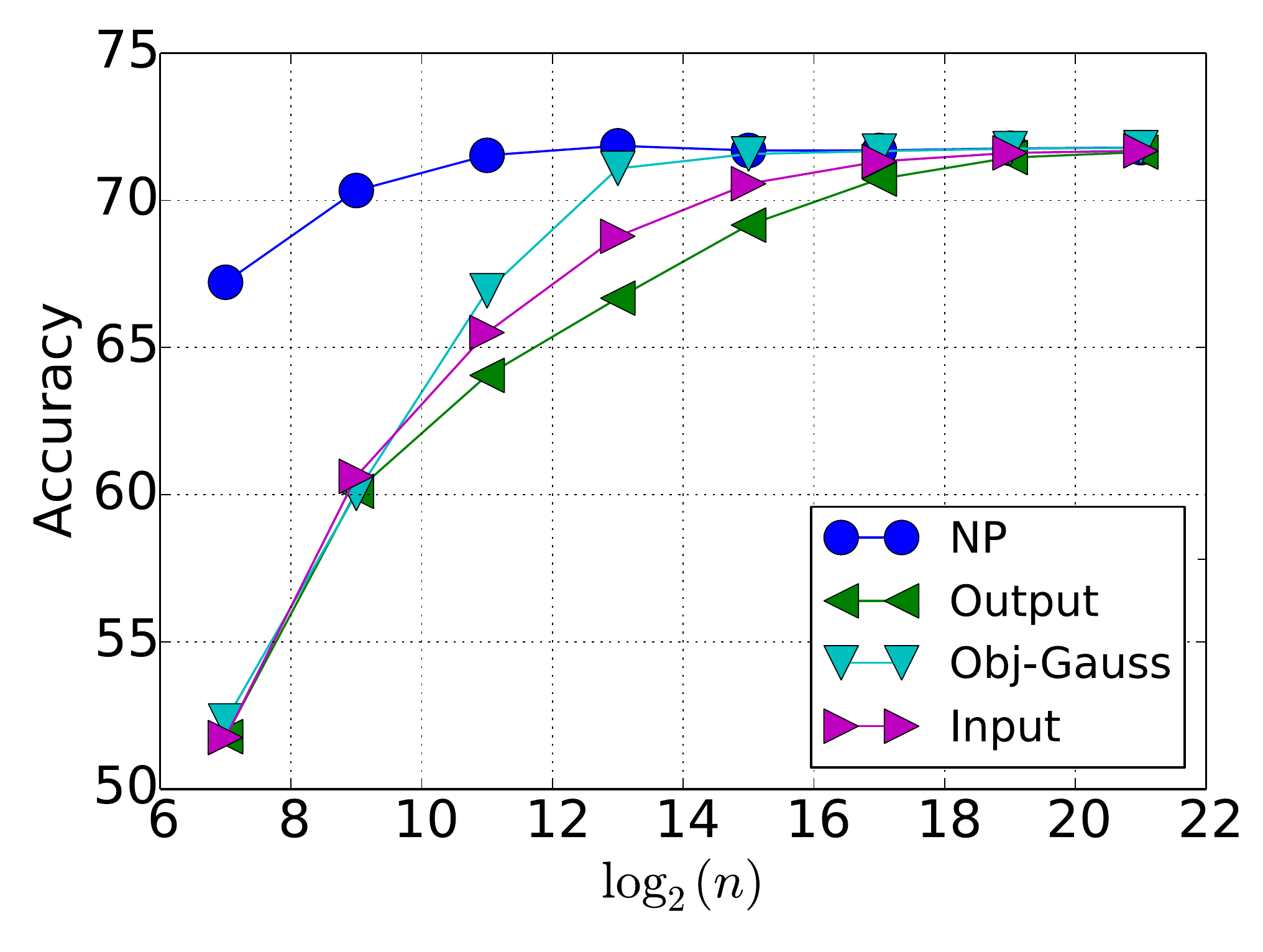}}
	\caption{Accuracy of differentially private logistic regression. The results were averaged over 100 trials while changing the example size $n$. We compared input perturbation~(Input), output perturbation with Laplace mechanism~(Output), and objective perturbation with Gaussian mechanism~(Obj-Gauss), and non-private linear regression as the baseline~(NP). }
	\label{fig:logistic_result}
\end{figure}

\subsection{Dataset and Preprocessing}
We used a dataset from \textit{Integrated Public Use Microdata Series:} (IPMS)~\citep{ipums:2014}, which contains $3,833,838$ census records collected in the US for year 2000 after removing unknown values and missing values. We performed an experiment with $n$, the size of the training dataset, by each $4$ times from $2^7(=128)$ to $2^{21}(=2,097,152)$. We set privacy parameters $\epsilon=\{0.1, 1.0\}, \delta=0.01$.

The IPMS dataset originally contained $13$ features. The binary status \textit{Marital status} was transformed into two attributes: \textit{Is Single} and \textit{Is Married}. Hence, $14$ attributes were employed. For linear regression model learning, \textit{Annual Income}, a continuous feature, was used as the prediction target. For logistic regression model learning, we converted \textit{Annual Income} into a binary attribute and used it as the label, in which values higher than a predefined threshold were mapped to $1$, and $0$ otherwise. In both kinds of model learning, the remaining attributes were used for the features. As preprocessing, we scaled the feature values so that the norm of each feature vector was at most $1$; {\em annual income} used as the prediction target was scaled so that the norm was at most $1$ before transformation to the binary label.

\subsection{Results}
Figure \ref{fig:linear_result} (a) and (b) show the experimental results of differentially private linear regression model learning. In Figure \ref{fig:linear_result}, the horizontal axis shows the logarithmic scale of the example size $n$, and the vertical axis shows the average RMSE of the comparative methods. As predicted by the theorem, the results show that the average RMSE of input perturbation approaches the RMSE of non-privacy linear regression as $n$ increases. Therefore, when the number of instances is very large, the performance of input perturbation is almost the same as that of non-privacy, as confirmed by Theorem \ref{th:excess_bound}.

Input perturbation is an approximation of objective perturbation with the Gaussian mechanism. So, at the limit of $n$, the behavior of input perturbation is equivalent to that of objective perturbation with the Gaussian mechanism. This can be confirmed from the results, too. Even with small $n$, we can see that the RMSEs of Input and Obj-Gauss are still quite close in both figures. This is because the difference of the excess risk of objective and input perturbation is in $O(1/n^{3/2})$.

\subsection{Differentially Private Logistic Regression Model Learning}

Figure \ref{fig:logistic_result} shows the experimental results of differentially private logistic regression model learning. In Figure \ref{fig:logistic_result}, the horizontal axis shows the logarithmic scale of the example size $n$, and the vertical axis shows the average accuracy of comparative methods. Similar to linear regression, the average accuracy of input perturbation is almost the same as that of objective perturbation with the Gaussian mechanism, when the example size $n$ is large. Because the average accuracy of input perturbation and objective perturbation approach the accuracy of non-privacy logistic regression as $n$ increases. However, when $n$ is small, the accuracy of input perturbation is slightly lower than that of objective perturbation with the Gaussian mechanism. This behavior can be caused by the approximation error of the logistic loss function.

\section{Conclusion}\label{sec:conclusion}
In this study, we propose a novel framework for differentially private ERM, input perturbation. In contrast to objective perturbation, input perturbation allows data contributors to take part in the process of privacy preservation of model learning. From the privacy analysis of the data releasing of the data contributors, the data collection process in the input perturbation satisfies $(O(\sqrt{n\epsilon},\delta)$-differential local privacy. Thus, from the perspective of data contributors, data collection with input perturbation can be preferable.

Models with randomized examples following the scheme of input perturbation are guaranteed to satisfy $(\epsilon, \delta)$-differential privacy. To achieve this approximation with randomization by independent data contributors, input perturbation requires that the loss function be quadratic with respect to the model parameter, $\vec{w}$. Applying other loss functions in our proposed method is remained as an area of our future work.

We compared the utility analysis and the empirical evaluation of input perturbation with those of output and objective perturbations in terms of the excess empirical risk against the non-privacy-preserving ERM. We show that the excess empirical risk of the model learned with input perturbation is $O(1/n)$, which is equivalent to that of objective perturbation in the optimal setting of $\sigma_u^2$ for every data contributors.

\subsubsection*{Acknowledgments.}

\bibliographystyle{plainnat}
\bibliography{paper}

\appendix

\section{Notation}\label{sec:notation}
Here, we summerize the notations in \cref{tab:notation}.
\begin{table}[H]
	\centering
	\caption{Table of notations}
	\label{tab:notation}
	\begin{tabular}{l|l}
		\hline
		Notation & Description \\
    \hline\hline
    $\dom{X} \subseteq \RealSet^d$ & domain of the $d$-dimentional feature vector \\
    $\dom{Y}$ & output domain \\
		$D \in \dom{Z}^n = (\dom{X}\times\dom{Z})^n$ & database of $n$ examples \\
		$z_i = (\vec{x}_i,y_i) \in \dom{Z}$ & the i-th example of data base $D$ \\
		$\vec{q}_i \in \RealSet^d$ & a vector computed by $\vec{q}(\vec{x}_i,y_i)$ \\
		$\vec{p}_i \in \RealSet^d$ & a vector computed by $\vec{p}(\vec{x}_i,y_i)$ \\
		$\vec{Q} \in \RealSet^{n\times d}$ & $[\vec{q}_1,\dots,\vec{q}_n]^T$ \\
		$\vec{p} \in \RealSet^d$ & $\sum_{i=1}^n\vec{p}_i$ \\
    \hline
    $\dom{W} \subseteq \RealSet^d$ & domain of the model parameter \\
		$\vec{w} \in \dom{W}$ & the model parameter \\
		$\eta$ & the upper bound of $\norm{\vec{w}}_2$ for any $\vec{w} \in \dom{W}$ \\
		$\ell:\dom{W}\times(\dom{X}\times\dom{Y})\to\RealSet$ & the loss function \\
		$\zeta$ & the upper bound of  $\norm{\nabla\ell}_2$ \\
		$\lambda$ & the upper bound of $\norm{\nabla^2\ell}_2$ \\
		$\hat{J}(\vec{w};D)$ & the average loss function \\
		$\hat{\vec{w}}$ & the optimal parameter of the average loss function \\
		$J(\vec{w};D)$ & the objective function of ERM \\
		$\vec{w}^*$ & the optimal parameter of ERM \\
    \hline
		$\epsilon,\delta$ & the differential privacy parameters \\
		$\sigma_b^2$ & the variance of Gaussian distribution \\
		$\sigma_u^2$ & the variance of Gaussian distribution \\
		$\tilde{\vec{q}}_i$ & $\vec{q}_i$ is added with noise from $\mathcal{N}(0,\frac{\sigma_u^2}{n})$ \\
		$\tilde{\vec{p}}_i$ & $\vec{p}_i$ is added with noise from $\mathcal{N}(0,\frac{\sigma_b^2}{n})$\\
		$\tilde{\vec{Q}}$ & $[\tilde{\vec{q}}_1,\dots,\tilde{\vec{q}}_n]^T$ \\
		$\tilde{\vec{p}}$ & $\sum_{i=1}^n\tilde{\vec{p}}_i$ \\
		$\tilde{D}$ & $D$ with noise added \\\hline
		$J^{out}(\vec{w};D)$ & the objective function of output perturbation \\
		$\vec{w}^{out}$ & the optimal parameter of output perturbation \\
		$J^{obj}(\vec{w};D)$ & the objective function of objective perturbation \\
		$\vec{w}^{obj}$ & the optimal parameter of objective perturbation \\
		$J^{in}(\vec{w};\tilde{D})$ & the objective function of input perturbation \\
		$\vec{w}^{in}$ & the optimal parameter of input perturbation \\ \hline
	\end{tabular}
\end{table}
\section{Proof of Lemma \ref{lem:lem1}}\label{app:regularization_bound}
We first introduce known results in order to prove this lemma.
\setcounter{lemma}{2}
\begin{lemma}[\citet{rao:2009}] \label{lm:wishart_dist}
  Let $\vec{Z}$ be a $p\times p$ random matrix drawn from Wishart distribution $\vec{Z} \sim \mathcal{W}_p(\vec{V},m)$ with $m$ degrees of freedom and variance matrix $\vec{V}$. Let $\vec{v}$ be a non-zero $p\times 1$ constant vector. Then,
	  \begin{align}
			\vec{v}^T\vec{Z}\vec{v} \sim \sigma_v^2\chi_m^2, \nonumber
		\end{align}
	where $\chi_m^2$ is the chi-squared distribution with $m$ degrees of freedom and $\sigma_v^2 = \vec{v}^TV\vec{v}$.  (Note that $\sigma_v^2$ is a constant; it is positive because $V$ is positive definite).
\end{lemma}
\begin{lemma}[\citet{laurent:2000}] \label{lm:chi_bound}
  Let $Z \sim\chi_m^2$. Then, for any $t > 0$,
	\begin{align}
		\p{Z - m \geq 2\sqrt{mt}+2t} \leq \exp(-t).
		\label{eq:upper_bound}
	\end{align}
  Also, for any $t > 0$,
	\begin{align}
		\p{m - Z \geq 2\sqrt{mt}} \leq \exp(-t).
		\label{eq:lower_bound}
	\end{align}
\end{lemma}
\begin{lemma}\label{lm:gauss_bound}
	Let $Z\sim\mathcal{N}(0,1)$, then for all $t>1$, we have
	\begin{align}
		\p{\abs{Z}>t}\leq e^{-t^2/2}.\nonumber
	\end{align}
\end{lemma}
From here, we prove Lemma \ref{lm:mu_bound} and Lemma \ref{lm:zero_bound} to prove Lemma \ref{lem:lem1}. We first give the tail bound of $\vec{w}^T\vec{U}^T\vec{U} \vec{w}$.
\begin{lemma}\label{lm:mu_bound}
Let $\vec{U} = [\vec{u}_1,\cdots,\vec{u}_n]^T$, where $\vec{u}_i \sim \mathcal{N}(0, \frac{\sigma_u^2}{n}\mathbb{I}_d)$, and $\gamma > 0$. Then, with probability at least $1-\gamma/2$ we get the following bound:
\begin{align}
 \sigma_u^2\left(1 - 2\sqrt{\frac{\log(4/\gamma)}{n}}\right)\norm{\vec{w}}_2^2 \leq \vec{w}^T\vec{U}^T\vec{U}\vec{w}
 \leq \sigma_u^2\left(1 + 2\sqrt{\frac{\log(4/\gamma)}{n}} + 2 \frac{\log(4/\gamma)}{n}\right)\norm{\vec{w}}_2^2  \nonumber
\end{align}
\end{lemma}
\begin{proof}
$\vec{U}^T\vec{U} \sim \mathcal{W}_d(\frac{\sigma_u^2}{n}\mathbb{I}, n)$ holds because $\vec{u}_i = (u_{i1}, \cdots, u_{id})^T \sim \mathcal{N}(0, \frac{\sigma_u^2}{n}\mathbb{I})$. By using Lemma \ref{lm:wishart_dist}, we thus get $\vec{w}^T\vec{U}^T\vec{U}\vec{w} \sim \frac{\sigma_u^2}{n}\norm{\vec{w}}_2^2 \chi_n^2$.
Noting that $\vec{w}^T\vec{U}^T\vec{U}\vec{w} \sim \frac{\sigma_u^2}{n}\norm{\vec{w}}_2^2Z$, the upper bound of $\vec{w}^T\vec{U}^T\vec{U}\vec{w}$ is derived by applying Eq. (\ref{eq:upper_bound}) of Lemma \ref{lm:chi_bound} as follows:
\begin{align}
		 & \p{Z - n \geq 2\sqrt{nt} + 2t}  \nonumber \\
		=& \p{\frac{Z}{n} \geq 1 + 2\sqrt{\frac{t}{n}} + 2 \frac{t}{n}} \nonumber \\
		=& \p{\frac{\sigma_u^2}{n}\norm{\vec{w}}_2^2Z \geq \sigma_u^2\norm{\vec{w}}_2^2\left(1 + 2\sqrt{\frac{t}{n}} + 2 \frac{t}{n}\right)} \nonumber \\
		=& \p{\vec{w}^T\vec{U}^T\vec{U}\vec{w} \geq \sigma_u^2\norm{\vec{w}}_2^2\left(1 + 2\sqrt{\frac{t}{n}} + 2 \frac{t}{n}\right)} \leq \exp(-t).
		\label{eq:mid1}
\end{align}
In a similar manner, by applying Eq. (\ref{eq:lower_bound}) of Lemma \ref{lm:chi_bound},
the lower bound of $\vec{w}^T\vec{U}^T\vec{U}\vec{w}$ is given as follows:
	\begin{align}
		 &\p{n - Z \geq 2\sqrt{nt}} \nonumber \\
		=&\p{\frac{Z}{n} \leq 1 - 2\sqrt{\frac{t}{n}}} \nonumber \\
		=&\p{\frac{\sigma_u^2}{n}\norm{\vec{w}}_2^2Z \leq \sigma_u^2\norm{\vec{w}}_2^2\left(1 - 2\sqrt{\frac{t}{n}}\right)}  \nonumber \\
		=&\p{\vec{w}^T\vec{U}^T\vec{U}\vec{w} \leq \sigma_u^2\norm{\vec{w}}_2^2\left(1 - 2\sqrt{\frac{t}{n}}\right)} \leq \exp(-t).
		\label{eq:mid2}
	\end{align}
	By setting $\frac{\gamma}{4} = \exp(-t)$, then we get $t = \log(\frac{4}{\gamma})$. Replacing the value of $t$ as $t = \log(\frac{4}{\gamma})$  and combining Eq. (\ref{eq:mid1}) and Eq. (\ref{eq:mid2}) gives our claim.
\end{proof}
Next, we investigate the tail bound of $\vec{w}^T(\vec{Q}^T\vec{U} + \vec{U}^T\vec{Q})\vec{w}$.
\begin{lemma}\label{lm:zero_bound}
Let $\vec{U} = [\vec{u}_1,\cdots,\vec{u}_n]^T$, where $\vec{u}_i \sim \mathcal{N}(0, \frac{\sigma_u^2}{n}\mathbb{I}_d)$. For $\gamma \in (0, 1]$, with probability at least $1-\frac{\gamma}{2}$,
\begin{align}
 - 2\sqrt{2d}\lambda\sigma_u\sqrt{\frac{\log(2/\gamma)}{n}}\norm{\vec{w}}_2^2
 \leq \vec{w}^T(\vec{U}^T\vec{Q} + \vec{Q}^T\vec{U})\vec{w}
 \leq 2\sqrt{2d}\lambda\sigma_u\sqrt{\frac{\log(2/\gamma)}{n}}\norm{\vec{w}}_2^2. \nonumber
\end{align}
\end{lemma}
\begin{proof}
\sloppy Let $\vec{v} = \vec{Q}\vec{w}$. Since $\vec{w}^T\vec{Q}^T\vec{U}\vec{w} = \vec{w}^T\vec{U}^T\vec{Q}\vec{w}$, we have $\vec{w}^T(\vec{Q}^T\vec{U} + \vec{U}^T\vec{Q})\vec{w} = 2\vec{w}^T\vec{Q}^T\vec{U}\vec{w} = 2\vec{v}^T\vec{U}\vec{w}$. From the property of the sum of the normally distributed independent random variables, we have $2\vec{v}^T\vec{U}\vec{w} \sim \mathcal{N}(0, 4\frac{\sigma_u^2}{n}\norm{\vec{v}\vec{w}^T}_F^2)$.
 Since  $\norm{\vec{v}\vec{w}^T}_F^2 = \norm{\vec{Q}\vec{w}\vec{w}^T}_F^2 =
 \norm{\vec{Q}}_F^2\norm{\vec{w}}_2^4$ holds, we get $\vec{w}^T(\vec{Q}^T\vec{U} + \vec{U}^T\vec{Q})\vec{w} \sim
 \mathcal{N}(0, 4\frac{\sigma_u^2}{n}\norm{\vec{Q}}_F^2\norm{\vec{w}}_2^4)$.
$\vec{w}^T(\vec{Q}^T\vec{U} + \vec{U}^T\vec{Q})\vec{w}$.
Application of Lemma \ref{lm:gauss_bound} thus yields
 \begin{align}
 \p{\abs{\vec{w}^T(\vec{Q}^T\vec{U} +
  \vec{U}^T\vec{Q})\vec{w}}>2\frac{\sigma_u}{\sqrt{n}}\norm{\vec{Q}}_F\norm{\vec{w}}_2^2t}\leq
  e^{-t^2/2}.
\end{align}
 By setting $\frac{\gamma}{2} = \exp(-\frac{t^2}{2})$, then we get $t = \sqrt{2\log(\frac{2}{\gamma})}$. To make sure $t\geq 1$, we need to have $\gamma \leq \frac{2}{\sqrt{e}}$. This is always true for any
 $\gamma \in (0, 1]$. Replacing the value of $t$, with probability at least $1-\frac{\gamma}{2}$ we get the following bound.
\begin{align}
 \abs{\vec{w}^T(\vec{Q}^T\vec{U} + \vec{U}^T\vec{Q})\vec{w}} \le
 2\sqrt{2}\sigma_u\norm{\vec{Q}}_F\sqrt{\frac{\log(2/\gamma)}{n}}\norm{\vec{w}}_2^2. \nonumber
\end{align}
We get the claim since $\norm{\vec{Q}}_F \le \sqrt{d}\lambda$.
\end{proof}
\begin{proof}[Proof of Lemma \ref{lem:lem1}]
By combining the bounds of Lemma \ref{lm:mu_bound} and Lemma \ref{lm:zero_bound}, with probability at least $1-\gamma$ we have the following bound:
\begin{align}
	&\sigma_u^2\left(1 - 2\sqrt{\frac{\log(4/\gamma)}{n}}\right)\norm{\vec{w}}_2^2  - 2\sqrt{2d}\lambda\sigma_u\sqrt{\frac{\log(2/\gamma)}{n}}\norm{\vec{w}}_2^2 \nonumber \\
	&\leq \vec{w}^T(\vec{U}^T\vec{U} + \vec{Q}^T\vec{U} + \vec{U}^T\vec{Q})\vec{w} \nonumber \\
	&\leq \sigma_u^2\left(1 + 2\sqrt{\frac{\log(4/\gamma)}{n}} + 2 \frac{\log(4/\gamma)}{n}\right)\norm{\vec{w}}_2^2 + 2\sqrt{2d}\sigma_u\lambda\sqrt{\frac{\log(2/\gamma)}{n}}\norm{\vec{w}}_2^2 \nonumber
\end{align}
The lemma holds by letting $\Delta_0 \vec{w}^T\vec{w} = \vec{w}^T(\vec{U}^T\vec{U}+\vec{U}^T\vec{Q}+\vec{Q}^T\vec{U})\vec{w}$.
\end{proof}
Here is a corollary of Lemma \ref{lem:lem1}
\begin{corollary}\label{lm:regular_cond}
 Let $\vec{U} = [\vec{u}_1,\cdots,\vec{u}_n]^T$, where $\vec{u}_i \sim \mathcal{N}(0, \frac{\sigma_u^2}{n}\mathbb{I}_d)$. Let $\Delta_0 \vec{w}^T\vec{w} = \vec{w}^T(\vec{U}^T\vec{U}+\vec{U}^T\vec{Q}+\vec{Q}^T\vec{U})\vec{w}$. Then for any $\gamma > 0$, we get the following:
\begin{multline}
	\sigma_u \ge  \frac{\sqrt{2d}\lambda\sqrt{\frac{\log(2/\gamma)}{n}} + \sqrt{2d\lambda^2 \frac{\log(2/\gamma)}{n}+ \frac{2\lambda}{\epsilon}\left(1-2\sqrt{\frac{\log(4/\gamma)}{n}}\right)}}{1-2\sqrt{\frac{\log(4/\gamma)}{n}}} \\
   \implies \frac{2\lambda}{\epsilon} \le \Delta_0 \le \kappa(n,\gamma)\textrm{  w.p. at least  } 1-\gamma\nonumber
\end{multline}
\end{corollary}
\begin{proof}
	We solve the following inequation,
	\begin{align}
		\varkappa(n,\gamma) &\ge \frac{2\lambda}{\epsilon}\nonumber\\
		\sigma_u^2\left(1 - 2\sqrt{\frac{\log(4/\gamma)}{n}}\right)  - &2\sqrt{2d}\lambda\sigma_u\sqrt{\frac{\log(2/\gamma)}{n}} \ge \frac{2\lambda}{\epsilon} \nonumber
	\end{align}
	Then, one of results of the above inequation is $\sigma_u$ as the left side of Corollary \ref{lm:regular_cond}.
	Hence, with $\sigma_u$ as the left side of Corollary \ref{lm:regular_cond} we have
	\begin{align}
		\varkappa(n,\gamma) \ge \frac{2\lambda}{\epsilon} \textrm{ w.p. at least } 1 \label{eq:col1}
	\end{align}
	From Lemma \ref{lem:lem1}, we have
	\begin{align}
		\varkappa(n,\gamma) \le \Delta_0 \le \kappa(n,\gamma)  \textrm{ w.p at least } 1 - \gamma \label{eq:col2}
	\end{align}
	Therefore from Eq. \ref{eq:col1} and Eq. \ref{eq:col2} with $\sigma_u$ as the left side of Corollary \ref{lm:regular_cond} we have
	\begin{align}
		\frac{2\lambda}{\epsilon} \le \Delta_0 \le \kappa(n,\gamma) \textrm{ w.p. at least } 1 - \gamma\nonumber
	\end{align}
\end{proof}
\begin{corollary}\label{col:sigma_u}
	When $\sigma_u =  \frac{\sqrt{2d}\lambda\sqrt{\frac{\log(2/\gamma)}{n}} + \sqrt{2d\lambda^2 \frac{\log(2/\gamma)}{n}+ \frac{2\lambda}{\epsilon}\left(1-2\sqrt{\frac{\log(4/\gamma)}{n}}\right)}}{1-2\sqrt{\frac{\log(4/\gamma)}{n}}}$, we have
	\begin{itemize}
		\item the upper bound of $\sigma_u$ is $\sigma_u \le \left(4\sqrt{2d}\lambda+4\sqrt{\frac{2\lambda}{\epsilon}}\right)\sqrt{\frac{\log\frac{4}{\gamma}}{n}} + \sqrt{\frac{2\lambda}{\epsilon}}$ with $n\ge 16\log\frac{4}{\gamma}$
		\item the lower bound of $\sigma_u$ is $\sigma_u \ge \sqrt{2d}\lambda\sqrt{\frac{\log\frac{2}{\gamma}}{n}} + \sqrt{\frac{2\lambda}{\epsilon}} $.
	\end{itemize}
\end{corollary}
\begin{proof}
	We derive the lower bound of $\sigma_u$ as
	\begin{align}
		\sigma_u &\ge \frac{\sqrt{2d}\lambda\sqrt{\frac{\log(2/\gamma)}{n}} + \sqrt{\frac{2\lambda}{\epsilon}\left(1-2\sqrt{\frac{\log(4/\gamma)}{n}}\right)}}{1-2\sqrt{\frac{\log(4/\gamma)}{n}}} \nonumber\\
		&\ge \frac{\sqrt{2d}\lambda\sqrt{\frac{\log(2/\gamma)}{n}}}{1-2\sqrt{\frac{\log(4/\gamma)}{n}}} + \frac{\sqrt{\frac{2\lambda}{\epsilon}}}{\sqrt{1-2\sqrt{\frac{\log(4/\gamma)}{n}}}} \nonumber \\
		&\ge \sqrt{2d}\lambda\sqrt{\frac{\log(2/\gamma)}{n}} + \sqrt{\frac{2\lambda}{\epsilon}} \nonumber
	\end{align}
	We derive the upper bound of $\sigma_u$ as
	\begin{align}
		\sigma_u &\le \frac{\sqrt{2d}\lambda\sqrt{\frac{\log(2/\gamma)}{n}} + \sqrt{2d\lambda^2 \frac{\log(2/\gamma)}{n}+ \frac{2\lambda}{\epsilon}}}{1-2\sqrt{\frac{\log(4/\gamma)}{n}}} \nonumber \\
		&\le \frac{\sqrt{2d}\lambda\sqrt{\frac{\log(2/\gamma)}{n}} + \sqrt{2d\lambda^2 \frac{\log(2/\gamma)}{n}}+ \sqrt{\frac{2\lambda}{\epsilon}}}{1-2\sqrt{\frac{\log(4/\gamma)}{n}}} -\sqrt{\frac{2\lambda}{\epsilon}} + \sqrt{\frac{2\lambda}{\epsilon}}\nonumber \\
		&\le \frac{2\sqrt{2d}\lambda\sqrt{\frac{\log(4/\gamma)}{n}} + 2\sqrt{\frac{2\lambda}{\epsilon}}\sqrt{\frac{\log(4/\gamma)}{n}}}{1-2\sqrt{\frac{\log(4/\gamma)}{n}}}  + \sqrt{\frac{2\lambda}{\epsilon}}\nonumber \\
		&\le \frac{2\sqrt{2d}\lambda + 2\sqrt{\frac{2\lambda}{\epsilon}}}{\sqrt{\frac{n}{\log(4/\gamma)}}-2}  + \sqrt{\frac{2\lambda}{\epsilon}}\nonumber
	\end{align}
	Letting $n\ge16\log\frac{4}{\gamma}$, we have  $\sqrt{\frac{n}{\log(4/\gamma)}}-2 > \frac{1}{2}\sqrt{\frac{n}{\log(4/\gamma)}}$. Hence,
	\begin{align}
		\sigma_u &\le \frac{2\sqrt{2d}\lambda + 2\sqrt{\frac{2\lambda}{\epsilon}}}{\frac{1}{2}\sqrt{\frac{n}{\log(4/\gamma)}}}  + \sqrt{\frac{2\lambda}{\epsilon}}\nonumber \\
		&\le \left( 4\sqrt{2d}\lambda + 4\sqrt{\frac{2\lambda}{\epsilon}} \right)\sqrt{\frac{\log(4/\gamma)}{n}}  + \sqrt{\frac{2\lambda}{\epsilon}}\nonumber
	\end{align}
\end{proof}
\section{Proof of Theorem \ref{th:proposed_dp}}\label{sec:proof-th-proposed-dp}
\begin{proof}[Proof of Theorem \ref{th:proposed_dp}]
 The objective function of input perturbation method is rearranged as follows:
\begin{align}
 J^{in}(\vec{w};\tilde{D}) &= \frac{1}{n}\sum_{i=1}^n\left(\frac{1}{2}\vec{w}^T\tilde{\vec{q}}_i\tilde{\vec{q}}_i^T\vec{w} - \tilde{\vec{p}}_i^T\vec{w} + s_i\right) + \frac{\Delta_{in}}{2n}\norm{\vec{w}}_2^2 \nonumber\\
 &= \frac{1}{n} \left( \frac{1}{2} \vec{w} \tilde{\vec{Q}}^T\tilde{\vec{Q}} \vec{w} - \tilde{\vec{p}}^T \vec{w}+ s \right) +  \frac{\Delta - \frac{2\lambda}{\epsilon}  }{2n} \norm{\vec{w}}_2^2 \nonumber\\
 &= \frac{1}{n} \left( \frac{1}{2} \vec{w} (\vec{Q}+\vec{U})^T(\vec{Q}+\vec{U}) \vec{w} - (\vec{p}-\vec{b})^T \vec{w}+ s  \right) +  \frac{\Delta - \frac{2\lambda}{\epsilon}  }{2n} \norm{\vec{w}}_2^2 \nonumber\\
 &= \begin{multlined}[t]
  \frac{1}{n} \left( \frac{1}{2} \vec{w} \vec{Q}^T\vec{Q}\vec{w} - \vec{p}^T \vec{w}+ s  \right) + \frac{\vec{b}^T \vec{w}}{n} \\ + \frac{1}{2n}\vec{w}^T(\vec{U}^T\vec{U}+\vec{U}^T\vec{Q}+\vec{Q}^T\vec{U} )\vec{w} + \frac{\Delta - \frac{2\lambda}{\epsilon}  }{2n} \norm{\vec{w}}_2^2
 \end{multlined} \nonumber\\
 &= \begin{multlined}[t]
  \frac{1}{2n}\sum_i \ell(f(\vec{w}, \vec{x}_i),y_i) + \frac{\vec{b}^T\vec{w}}{n} \\ + \frac{1}{2n}\vec{w}^T(\vec{U}^T\vec{U}+\vec{U}^T\vec{Q}+\vec{Q}^T\vec{U} )\vec{w}+\frac{\Delta - \frac{2\lambda}{\epsilon}}{2n} \norm{\vec{w}}_2^2
 \end{multlined} \nonumber\\
 &= \frac{1}{2n}\sum_i \ell(f(\vec{w}, \vec{x}_i),y_i) + \frac{\vec{b}^T\vec{w}}{n} +\frac{1}{2n}(\Delta_0+\Delta - \frac{2\lambda}{\epsilon})\norm{\vec{w}}_2^2 \label{eq:finalobj}
\end{align}
 where $\Delta_0\vec{w}^T\vec{w} = \vec{w}^T (\vec{U}^T \vec{U} + \vec{U}^T\vec{Q} + \vec{Q}^T\vec{U}) \vec{w}$. Noting that $\vec{b}$ follows Gaussian distribution $\mathcal{N}(0, \sigma_b^2\mathbb{I}_{d\times d})$, Eq.~\ref{eq:finalobj} is equivalent to the objective function of the objective perturbation method with the Gaussian mechanism except the regularization parameter. Let $\vec{\alpha} = \argmin_{\vec{w}\in {\cal W}} J^{in}$. Then, from Theorem 2 in \citep{kifer2012private},  $\vec{\alpha}$ is $(\epsilon, \delta')$-differentially private if the following conditions are true:
 \begin{align}
	\Delta + \Delta_0 - \frac{2\lambda}{\epsilon} \geq \frac{2\lambda}{\epsilon} \textrm{ and } \sigma_b^2=\frac{\zeta^2(8\log(2/\delta')+4\epsilon)}{\epsilon^2} \nonumber
	\end{align}
	The second condition on $\sigma_b^2$ is always satisfied by the parameter setting in \cref{al:safe_sharing} of input perturbation. The first condition is equivalent to $\Delta_0 \ge
 \frac{2\lambda}{\epsilon}$ because $\Delta > \frac{2\lambda}{\epsilon}$ holds by assumption. From Lemma \ref{lm:regular_cond}, $\Delta_0 \ge \frac{2\lambda}{\epsilon}$ holds w.p. at least $1-\gamma$ if $\sigma_u^2$ is set as in \cref{al:safe_sharing}. Thus, the output of \cref{al:proposed} satisfies $(\epsilon,\delta')$-differential privacy w.p. at least $1-\gamma$. This can be transformed into a deterministic statement by using a proof technique of Theorem 2 in \citep{kifer2012private}. Let $\textbf{good}_u$ be the set $\{\vec{U}\in\RealSet^{n\times d}\vert \Delta_0 \geq \frac{2\lambda}{\epsilon}\}$. Then, from the definition of differential privacy, we have
\begin{multline}
	e^{-\epsilon}(\p{\vec{w}^{in}=\vec{\alpha}\vert \vec{U} \in \textbf{good}_u;D'}-\delta') \leq \\ \p{\vec{w}^{in}=\vec{\alpha}\vert \vec{U} \in \textbf{good}_u;D} \leq e^{\epsilon}\p{\vec{w}^{in}=\vec{\alpha}\vert \vec{U} \in \textbf{good}_u;D'}+\delta' \nonumber
\end{multline}
where $\p{\vec{U} \in \textbf{good}_u}= 1 - \gamma$. With this, we get the following.
\begin{align}
	&\p{\vec{w}^{in}=\vec{\alpha};D} \nonumber \\
	&= \begin{multlined}[t]
   \p{\vec{w}^{in}=\vec{\alpha}\vert\vec{U} \in \textbf{good}_u;D}\p{\vec{U} \in \textbf{good}_u} \\ + \p{\vec{w}^{in}=\vec{\alpha}\vert\vec{U} \in \overline{\textbf{good}_u};D}\p{\vec{U} \in \overline{\textbf{good}_u}}
  \end{multlined} \nonumber\\
	&\leq \left(e^{\epsilon}\p{\vec{w}^{in}=\vec{\alpha}\vert\vec{U} \in \textbf{good}_u;D'} + \delta' \right)\p{\vec{U} \in \textbf{good}_u} + \gamma  \nonumber \\
	&\leq e^{\epsilon}\p{\vec{w}^{in}=\vec{\alpha}\vert\vec{U} \in \textbf{good}_u;D'}\p{\vec{U} \in \textbf{good}_u} + \delta' \p{\vec{U} \in \textbf{good}_u} + \gamma\nonumber \\
	&\leq e^{\epsilon}\p{\vec{w}^{in}=\vec{\alpha};D'} + \delta'(1-\gamma) +  \gamma\nonumber\\
	&\leq e^{\epsilon}\p{\vec{w}^{in}=\vec{\alpha};D'} + \delta' + \gamma - \delta'\gamma \nonumber\\
	&\leq e^{\epsilon}\p{\vec{w}^{in}=\vec{\alpha};D'} + \delta' + \gamma \nonumber
\end{align}
Letting $\gamma = \frac{\delta}{2}$ and $\delta'=\frac{\delta}{2}$ we have
\begin{eqnarray}
 \p{\vec{w}^{in}=\vec{\alpha};D} \leq e^{\epsilon}\p{\vec{w}^{in}=\vec{\alpha};D'} +  \delta, \nonumber
\end{eqnarray}
which concludes the proof.

\end{proof}

\section{Proof of Theorem \ref{th:excess_bound}}\label{sec:th-excess-bound}
We first show some lemmas for Theorem \ref{th:excess_bound}.
 \begin{lemma}\label{lem:8}
	 Suppose $\Delta_0\ge 0$. Let $J^{\#}(\vec{w} ;D) = \sum_i \ell(f(\vec{w}, \vec{x}_i),y_i)+\frac{\Delta+\Delta_0-\frac{2\lambda}{\epsilon}}{2n}$ and the minimizer be $\vec{w}^{\#} = \mbox{arg}\min_{\vec{w} \in {\cal\ell W}} J^{\#}(\vec{w} ;D)$. Let $\vec{w}^{in}$ be the output of \cref{al:proposed}. Then,
  \begin{align}
    \| \vec{w}^{\#} - \vec{w}^{in} \|_2 \le \frac{ 2 \| \vec{b}\|}{ \Delta + \Delta_0  - \frac{2\lambda}{\epsilon}  } \nonumber
  \end{align}
 \end{lemma}
\begin{proof}
 From Eq. \ref{eq:finalobj}, $J^{in}$ is $\frac{\Delta + \Delta_0  - \frac{2\lambda}{\epsilon}}{2n}$-strong convex. Thus, we have
 \begin{align}
  J^{in}(\vec{w}^{\#};\tilde{D})
   \ge& J^{in}(\vec{w}^{in};\tilde{D}) + \frac{\Delta + \Delta_0  - \frac{2\lambda}{\epsilon}}{2n} \| \vec{w}^{\#} - \vec{w}^{in} \|_2^2 \nonumber\\
  J^{\#}(\vec{w}^{\#};D) + \frac{\vec{b}^T \vec{w^{\#}}}{n}
   \ge& J^{\#}(\vec{w}^{in};D)+ \frac{\vec{b}^T \vec{w^{in}}}{n} + \frac{\Delta + \Delta_0  - \frac{2\lambda}{\epsilon}}{2n} \| \vec{w}^{\#} - \vec{w}^{in} \|_2^2  \label{eq:forlem9}\\
   \ge&  J^{\#}(\vec{w}^{\#};D)+ \frac{\vec{b}^T \vec{w^{in}}}{n} + \frac{\Delta + \Delta_0  - \frac{2\lambda}{\epsilon}}{2n} \| \vec{w}^{\#} - \vec{w}^{in} \|_2^2\nonumber\\
    \vec{b}^T (\vec{w}^{\#} - \vec{w}^{in}) \ge& \frac{\Delta + \Delta_0  - \frac{2\lambda}{\epsilon}}{2} \| \vec{w}^{\#} - \vec{w}^{in} \|_2^2\nonumber\\
    \|\vec{b}\|_2 \| \vec{w}^{\#} - \vec{w}^{in} \|_2 \ge& \frac{\Delta + \Delta_0  - \frac{2\lambda}{\epsilon}}{2} \| \vec{w}^{\#} - \vec{w}^{in} \|_2^2\nonumber\\
    \| \vec{w}^{\#} - \vec{w}^{in} \|_2 \le& \frac{ 2 \| \vec{b}\|}{ \Delta + \Delta_0  - \frac{2\lambda}{\epsilon}  } \nonumber
 \end{align}
To have Eq. \ref{eq:forlem9}, we used that fact that $\vec{w}^{\#}$ is the minimizer of $J^{\#}$.

\end{proof}
\begin{lemma}\label{lem:9}
 Suppose $\Delta_0 \ge 0$. Let $J^{\#}(\vec{w} ;D) = \sum_i \ell(f(\vec{w}, \vec{x}_i),y_i)+\frac{\Delta+\Delta_0-\frac{2\lambda}{\epsilon}}{2n}$ and the minimizer be $\vec{w}^{\#} = \mbox{arg}\min_{\vec{w} \in {\cal\ell W}} J^{\#}(\vec{w} ;D)$. Let $\vec{w}^{in}$ be the output of \cref{al:proposed}. Then,
 \begin{eqnarray}
  J^{\#}(\vec{w}^{in};D) - J^{\#}(\vec{w}^{\#};D) \le \frac{2\|\vec{b}\|_2^2}{n(\Delta + \Delta_0  - \frac{2\lambda}{\epsilon}  )} \nonumber
 \end{eqnarray}
\end{lemma}
\begin{proof}
  From Eq. \ref{eq:forlem9} in Lemma \ref{lem:8},
  \begin{eqnarray}
   J^{\#}(\vec{w}^{\#};D) + \frac{\vec{b}^T \vec{w^{\#}}}{n} &\ge&
   J^{\#}(\vec{w}^{in};D)+ \frac{\vec{b}^T \vec{w^{in}}}{n} + \frac{\Delta + \Delta_0  - \frac{2\lambda}{\epsilon}}{2n} \| \vec{w}^{\#} - \vec{w}^{in} \|_2^2  \nonumber\\
   J^{\#}(\vec{w}^{in};D) -  J^{\#}(\vec{w}^{\#};D)  &\le& \frac{\vec{b}^T(\vec{w^{\#}} - \vec{w^{in}} )}{n}  -  \frac{\Delta + \Delta_0  - \frac{2\lambda}{\epsilon}}{2n} \| \vec{w}^{\#} - \vec{w}^{in} \|_2^2  \nonumber\\
   &\le& \frac{  \|\vec{b} \|_2  \|  \vec{w^{\#}} - \vec{w^{in}}\|_2}{n} \nonumber\\
   &\le& \frac{ 2 \|\vec{b} \|_2^2     }{n   (\Delta + \Delta_0  - \frac{2\lambda}{\epsilon})  } \label{eq:18}
  \end{eqnarray}
  To have Eq. \ref{eq:18}, we used the result of Lemma \ref{lem:8}.

\end{proof}
\begin{lemma}\label{lem:lem10}
 Suppose $\Delta_0 \ge 0$. Let $\hat{\vec{w}}$ be the minimizer of $\hat{J}(\vec{w};D)$ and let $\vec{w}^{in}$ be the output of \cref{al:proposed}. Then,
 \begin{eqnarray}
  \hat{J}(\vec{w}^{in};D) - \hat{J}(\hat{\vec{w}};D) \le \frac{2\|\vec{b}\|_2^2}{n(\Delta + \Delta_0  - \frac{2\lambda}{\epsilon}  )} + \frac{\Delta + \Delta_0-\frac{2\lambda}{\epsilon}}{2n} \| \hat{\vec{w}}\|_2^2 \label{eq:lem10}
 \end{eqnarray}
\end{lemma}
\begin{proof}
\begin{align}
	& \hat{J}(\vec{w}^{in};D) -   \hat{J}(\hat{\vec{w}};D) \nonumber\\
 \le&  (J^{\#}(\vec{w}^{in};D) -  J^{\#}(\vec{w}^{\#};D) ) +  (J^{\#}(\vec{w}^{\#};D) -  J^{\#}(\hat{\vec{w}};D) ) \nonumber \\
  &+ \frac{\Delta+\Delta_0-\frac{2\lambda}{\epsilon}}{2n} \| \hat{\vec{w}}\|_2^2 - \frac{  \Delta + \Delta_0  - \frac{2\lambda  }{\epsilon}}{n} \| \vec{w}^{in}\|_2^2 \nonumber\\
 \le&  (J^{\#}(\vec{w}^{in};D) -  J^{\#}(\vec{w}^{\#};D) )    + \frac{\Delta+\Delta_0-\frac{2\lambda}{\epsilon}}{2n} \| \hat{\vec{w}}\|_2^2 - \frac{  \Delta + \Delta_0  - \frac{2\lambda  }{\epsilon}}{n} \| \vec{w}^{in}\|_2^2 \label{eq:12}\\
 \le& \frac{2\|\vec{b}\|_2^2}{n(\Delta +	\Delta_0  - \frac{2\lambda}{\epsilon}  )}     + \frac{\Delta+\Delta_0-\frac{2\lambda}{\epsilon}}{2n} \| \hat{\vec{w}}\|_2^2 - \frac{  \Delta + \Delta_0  - \frac{2\lambda  }{\epsilon}}{n} \| \vec{w}^{in}\|_2^2 \label{eq:13}\\
 \le& \frac{2\|\vec{b}\|_2^2}{n(\Delta +	\Delta_0  - \frac{2\lambda}{\epsilon}  )}     + \frac{\Delta+\Delta_0-\frac{2\lambda}{\epsilon}}{2n} \| \hat{\vec{w}}\|_2^2  \nonumber
\end{align}
To have Eq. \ref{eq:12}, we used the fact that $J^{\#}(\vec{w}^{\#};D) - J^{\#}(\hat{\vec{w}};D) \le 0$. Eq. \ref{eq:13} is obtained by applying the result of Lemma \ref{lem:9}.

\end{proof}
\begin{proof}[Proof of Lemma \ref{lm:empirical_bound}]
We start from the upper bound Eq. \ref{eq:lem10} in Lemma \ref{lem:lem10}. From Lemma 28 in \citep{kifer2012private} (and Lemma 2 in \citep{dasgupta2007probabilistic}), we have w.p. at least $1-\beta$
\begin{align}
	\|\vec{b} \|_2 \le \sqrt{\frac{2d\zeta^2 (8\log \frac{4}{\delta} + 4\epsilon) \log \frac{1}{\beta}}{\epsilon^2}}.\label{eq:b}
\end{align}
From Corollary \ref{lm:regular_cond} with $\sigma_u^2$ in \cref{al:safe_sharing}, we have w.p. at least $1-\gamma$
\begin{align}
 \frac{2\lambda}{\epsilon}\le \Delta_0 \le \kappa(n,\gamma) .\label{eq:reg}
\end{align}
Substituting Eq. \ref{eq:b} and Eq. \ref{eq:reg} into Eq. \ref{eq:lem10}, we have w.p. at least $1-\beta -\gamma$
\begin{align}
 \hat{J}(\vec{w}^{in};D) - \hat{J}(\hat{\vec{w}};D)
 \le& \frac{4d\zeta^2 (8\log \frac{4}{\delta} + 4\epsilon) \log \frac{1}{\beta}}{n\epsilon^2 \Delta} + \frac{\Delta + \kappa(n,\gamma)-\frac{2\lambda}{\epsilon}}{2n} \| \hat{\vec{w}}\|_2^2\label{eq:17}
\end{align}
Substituting value of $\kappa(n,\gamma)$ in Lemma \ref{lm:regular_cond} into Eq.~\ref{eq:lem10} we have, w.p. at least $1-\beta-\gamma$.
\begin{align}
 \hat{J}(\vec{w}^{in};D) -    \hat{J}(\hat{\vec{w}};D)
 \le& \frac{4d\zeta^2 (8\log \frac{4}{\delta} + 4\epsilon) \log \frac{1}{\beta}}{n\epsilon^2 \Delta} + \frac{\Delta }{2n} \| \hat{\vec{w}}\|_2^2 \nonumber \\
 +& \frac{\sigma_u^2 + 2\sigma_u^2\sqrt{\frac{\log(4/\gamma)}{n}} + 2\sigma_u^2\frac{\log(4/\gamma)}{n} + 2\sqrt{2d}\lambda\sigma_u\sqrt{\frac{\log(2/\gamma)}{n}}}{2n} \| \hat{\vec{w}}\|_2^2 - \frac{\frac{2\lambda}{\epsilon}}{2n} \| \hat{\vec{w}}\|_2^2 \nonumber \\
 \le& \frac{4d\zeta^2 (8\log \frac{4}{\delta} + 4\epsilon) \log \frac{1}{\beta}}{n\epsilon^2 \Delta} + \frac{\Delta }{2n} \| \hat{\vec{w}}\|_2^2 \nonumber \\
 +& \frac{\sigma_u^2 - \frac{2\lambda}{\epsilon}}{2n} \| \hat{\vec{w}}\|_2^2 + \frac{\sigma_u^2\sqrt{\log\frac{4}{\gamma}} + \sigma_u^2\frac{\log\frac{4}{\gamma}}{\sqrt{n}} + \sigma_u\lambda\sqrt{2d\log\frac{2}{\gamma}}}{n\sqrt{n}} \| \hat{\vec{w}}\|_2^2 \label{eq:emp_risk_bound}
\end{align}

\end{proof}
\begin{proof}[Proof of Theorem \ref{th:excess_bound}]
	We have
	\begin{align}
		2\sqrt{\log\frac{4}{\gamma}} &\le \log\frac{4}{\gamma} + 1 = \log4 + \log\frac{1}{\gamma} + 1 \nonumber \\
		\log\frac{4}{\gamma} & = \log4 + \log\frac{1}{\gamma}\nonumber \\
		2\sqrt{\log\frac{2}{\gamma}} &\le \log\frac{2}{\gamma} + 1 = \log2 + \log\frac{1}{\gamma} + 1 \nonumber
	\end{align}
	By substituting above results into Eq. \ref{eq:emp_risk_bound}, then w.p. at least $1-\beta-\gamma$ we have
 \begin{align}
  &\hat{J}(\vec{w}^{in};D) - \hat{J}(\hat{\vec{w}};D) \nonumber\\
  \le& \begin{multlined}[t]
   \frac{4d\zeta^2 (8\log \frac{4}{\delta} + 4\epsilon) \log \frac{1}{\beta}}{n\epsilon^2 \Delta} + \frac{\Delta}{2n} \| \hat{\vec{w}}\|_2^2 + \frac{\sigma_u^2 - \frac{2\lambda}{\epsilon}}{2n} \| \hat{\vec{w}}\|_2^2 +\left(\sigma_u^2 + \frac{2\sigma_u^2}{\sqrt{n}} + \sqrt{2d}\lambda\sigma_u\right)\frac{\| \hat{\vec{w}}\|_2^2}{2n\sqrt{n}}\log\frac{1}{\gamma} \\
   + \left(\sigma_u^2 (\log4+1)+ \frac{2\sigma_u^2}{\sqrt{n}}\log4 + \sqrt{2d}\lambda\sigma_u (\log2+1)\right)\frac{\| \hat{\vec{w}}\|_2^2}{2n\sqrt{n}}
   \end{multlined} \nonumber\\
  \le& \begin{multlined}[t]
   \frac{4d\zeta^2 (8\log \frac{4}{\delta} + 4\epsilon) }{n\epsilon^2 \Delta} \log \frac{1}{\beta}+  \left(\sigma_u^2 + \frac{2\sigma_u^2}{\sqrt{n}} + \sqrt{2d}\lambda\sigma_u\right)\frac{\| \hat{\vec{w}}\|_2^2}{2n\sqrt{n}}\log\frac{1}{\gamma} \\
   + \frac{\Delta}{2n} \| \hat{\vec{w}}\|_2^2 + \frac{\sigma_u^2 - \frac{2\lambda}{\epsilon}}{2n}\| \hat{\vec{w}}\|_2^2+ \left(\sigma_u^2 (\log4+1)+ \frac{2\sigma_u^2}{\sqrt{n}}\log4 + \sqrt{2d}\lambda\sigma_u (\log2+1)\right)\frac{\| \hat{\vec{w}}\|_2^2}{2n\sqrt{n}}
  \end{multlined} \nonumber
\end{align}
Letting $\beta=\gamma=\nu$, w.p. at least $1-2\nu$
\begin{align}
	&\hat{J}(\vec{w}^{in};D) -    \hat{J}(\hat{\vec{w}};D) \le \left(\frac{4d\zeta^2 (8\log \frac{4}{\delta} + 4\epsilon) }{n\epsilon^2 \Delta} +  \left(\sigma_u^2 + \frac{2\sigma_u^2}{\sqrt{n}} + \sqrt{2d}\lambda\sigma_u\right)\frac{\| \hat{\vec{w}}\|_2^2}{2n\sqrt{n}}\right)\log\frac{1}{\nu}\nonumber \\
  &+ \left(\frac{\Delta}{2n} \| \hat{\vec{w}}\|_2^2 + \frac{\sigma_u^2 - \frac{2\lambda}{\epsilon}}{2n}\| \hat{\vec{w}}\|_2^2+ \left(\sigma_u^2 (\log4+1)+ \frac{2\sigma_u^2}{\sqrt{n}}\log4 + \sqrt{2d}\lambda\sigma_u (\log2+1)\right)\frac{\| \hat{\vec{w}}\|_2^2}{2n\sqrt{n}}\right)\nonumber \\
  &\le a \log\frac{1}{\nu} + b \nonumber
\end{align}
where
\begin{align}
	a &= \left(\frac{4d\zeta^2 (8\log \frac{4}{\delta} + 4\epsilon) }{n\epsilon^2 \Delta} +  \left(\sigma_u^2 + \frac{2\sigma_u^2}{\sqrt{n}} + \sqrt{2d}\lambda\sigma_u\right)\frac{\| \hat{\vec{w}}\|_2^2}{2n\sqrt{n}}\right) \nonumber\\
	b &= \left(\frac{\Delta}{2n} \| \hat{\vec{w}}\|_2^2 + \frac{\sigma_u^2 - \frac{2\lambda}{\epsilon}}{2n}\| \hat{\vec{w}}\|_2^2+ \left(\sigma_u^2 (\log4+1)+ \frac{2\sigma_u^2}{\sqrt{n}}\log4 + \sqrt{2d}\lambda\sigma_u (\log2+1)\right)\frac{\| \hat{\vec{w}}\|_2^2}{2n\sqrt{n}}\right)\label{eq:19}
\end{align}
By setting $s = a\log(1/\nu) + b$, we get $\nu = e^{b/a}e^{-s/a}$. Substitution this into Eq. \ref{eq:19} gives
\begin{align}
 \p{\hat{J}(\vec{w}^{in};D) -    \hat{J}(\hat{\vec{w}};D) \ge s} \le  2e^{b/a}e^{-s/a}.
\end{align}
From here, we compute the expectation of excess empirical risk by removing $\nu$ in the above equation.
\begin{align}
	\Mean{\hat{J}(\vec{w}^{in};D)-\hat{J}(\hat{\vec{w}};D)} &\le \xi + \int_{\xi}^{\infty} \p{\hat{J}(\vec{w}^{in}) -    \hat{J}(\hat{\vec{w}}) \ge s} \mathrm{d}s \nonumber \\
	&\le \xi + \int_{\xi}^{\infty} 2e^{b/a}e^{-s/a} \mathrm{d}s \nonumber \\
	&\le \xi + 2ae^{b/a}e^{-\xi/a} \nonumber
\end{align}
Using $\xi = b$, we have \begin{align}
	&\Mean{\hat{J}(\vec{w}^{in};D)-\hat{J}(\hat{\vec{w}};D)} \le 2a + b \nonumber\\
	&\le 2 \left(\frac{4d\zeta^2 (8\log \frac{4}{\delta} + 4\epsilon) }{n\epsilon^2 \Delta} +  \left(\sigma_u^2 + \frac{2\sigma_u^2}{\sqrt{n}} + \sqrt{2d}\lambda\sigma_u\right)\frac{\| \hat{\vec{w}}\|_2^2}{2n\sqrt{n}}\right)\nonumber \\
	&+  \left(\frac{\Delta}{2n} \| \hat{\vec{w}}\|_2^2 + \frac{\sigma_u^2 - \frac{2\lambda}{\epsilon}}{2n}\| \hat{\vec{w}}\|_2^2+ \left(\sigma_u^2 (\log4+1)+ \frac{2\sigma_u^2}{\sqrt{n}}\log4 + \sqrt{2d}\lambda\sigma_u (\log2+1)\right)\frac{\| \hat{\vec{w}}\|_2^2}{2n\sqrt{n}}\right)\nonumber \\
	&\le \left(\frac{8d\zeta^2 (8\log \frac{4}{\delta} + 4\epsilon) }{n\epsilon^2 \Delta} + \frac{\Delta}{2n} \| \hat{\vec{w}}\|_2^2\right) \nonumber \\
	&+ \frac{\sigma_u^2 - \frac{2\lambda}{\epsilon}}{2n}\| \hat{\vec{w}}\|_2^2 + \left(\sigma_u^2 \log(4e^3)+ \frac{2\sigma_u^2\log(4e^2)}{\sqrt{n}}+ \sqrt{2d}\lambda\sigma_u \log(2e^3)\right)\frac{\| \hat{\vec{w}}\|_2^2}{2n\sqrt{n}} \label{eq:21}
\end{align}
To get a tight bound of $\Mean{\hat{J}(\vec{w}^{in};D)-\hat{J}(\hat{\vec{w}};D)}$, we set $\Delta=\Theta\left(\frac{\sqrt{\zeta^2d\log\frac{1}{\delta}}}{\epsilon\norm{\hat{\vec{w}}}_2}\right)$. $\sigma_u$ is set as the lowest value specified in \cref{al:proposed}. Noting that $\sigma_u \le \left(4\sqrt{2d}\lambda + 2\sqrt{\frac{2\lambda}{\epsilon}}\right)\sqrt{\frac{\log(8/\delta)}{n}} + \sqrt{\frac{2\lambda}{\epsilon}}$ with $n\ge16\log\frac{8}{\delta}$ from Corollary \ref{col:sigma_u}, we have $\sigma_u = O(1)$. Hence, using these settings with Eq. \ref{eq:21}, we have the following:
\begin{align}
	\frac{8d\zeta^2 (8\log \frac{4}{\delta} + 4\epsilon) }{n\epsilon^2 \Delta} + \frac{\Delta}{2n} \| \hat{\vec{w}}\|_2^2 =& O\left(\frac{\zeta\norm{\hat{\vec{w}}}_2\sqrt{d\log(1/\delta)}}{\epsilon n}\right) \nonumber\\
	\frac{\sigma_u^2 - \frac{2\lambda}{\epsilon}}{2n}\| \hat{\vec{w}}\|_2^2 \le& O\left(\frac{1}{n\sqrt{n}}\right)\nonumber\\
	\left(\sigma_u^2 \log(4e^3)+ \frac{2\sigma_u^2\log(4e^2)}{\sqrt{n}}+ \sqrt{2d}\lambda\sigma_u \log(2e^3)\right)\frac{\| \hat{\vec{w}}\|_2^2}{2n\sqrt{n}} =& O\left(\frac{1}{n\sqrt{n}}\right) \nonumber
\end{align}
From these, we thus have
\begin{align}
	\Mean{\hat{J}(\vec{w}^{in};D)-\hat{J}(\hat{\vec{w}};D)} = O\left(\frac{\zeta\norm{\hat{\vec{w}}}_2\sqrt{d\log(1/\delta)}}{\epsilon n}\right) \nonumber
\end{align}
by setting $\Delta=\Theta\left(\frac{\sqrt{\zeta^2d\log\frac{1}{\delta}}}{\epsilon\norm{\hat{\vec{w}}}_2}\right)$ and $\sigma_u$ as the lowest value specified in \cref{al:safe_sharing} and $n\ge16\log\frac{8}{\delta}$.
\end{proof}

\section{Proof of Corollary \ref{cor:data-priv}}\label{sec:proof-thm-local-gaussian}
Adding a Gaussian noise into the released data satisfies $(\epsilon,\delta)$-differential local privacy as well as the Gaussian mechanism~\citep{dwork2014algorithmic}. We derive a differential local privacy version of the Gaussian mechanism as follows.
\begin{theorem}\label{thm:local-gaussian}
  Let $\mathcal{X} \subseteq \RealSet^d$ is a bounded domain of input such that $\|\vec{x} - \vec{x}'\|_2 \le B$ for any $\vec{x}, \vec{x}' \in \mathcal{X}$. Given $\epsilon \in (0,1)$ and $\delta > 0$, a mechanism $M$ outputs $M(x) = x + Z$ where $Z \sim \mathcal{N}(\vec{0}, \sigma^2\mat{I})$. If $\sigma \ge cB/\epsilon$ where $c > \sqrt{2\ln(1.25/\delta)}$, $M$ is $(\epsilon, \delta)$-differentially locally private.
\end{theorem}
\begin{proof}
  From the definition of the mechanism, we have for any $\vec{r} \in \RealSet^d$, $\vec{x}, \vec{x}' \in \mathcal{X}$
  \begin{align*}
    \p{M(\vec{x}) = \vec{r}} =& \p{\vec{x} + Z = \vec{r}} \\
     =& \frac{1}{\sqrt{(2\pi\sigma^2)^d}}\exp\left(-\frac{1}{2\sigma^2}\norm{\vec{r} - \vec{x}}_2^2\right) \\
     =& \frac{1}{\sqrt{(2\pi\sigma^2)^d}}\exp\left(-\frac{1}{2\sigma^2}\norm{\vec{r} - \vec{x}' + \vec{x}' - \vec{x}}_2^2\right) \\
     =& \frac{1}{\sqrt{(2\pi\sigma^2)^d}}\exp\left(-\frac{1}{2\sigma^2}\left(\norm{\vec{r} - \vec{x}'}_2^2  + \norm{\vec{x}' - \vec{x}}_2^2 + 2(\vec{r} - \vec{x}')^T(\vec{x}' - \vec{x})\right)\right) \\
     =& \exp\left(-\frac{1}{2\sigma^2}\left(\|\vec{x}' - \vec{x}\|_2^2 + 2(\vec{r} - \vec{x}')^T(\vec{x}' - \vec{x})\right)\right)\p{\vec{x}' + Z = \vec{r}} \\
     =& \exp\left(-\frac{1}{2\sigma^2}\left(2(\vec{r} - \vec{x})^T(\vec{x}' - \vec{x}) - \|\vec{x}' - \vec{x}\|_2^2\right)\right)\p{\vec{x}' + Z = \vec{r}} \\
     =& \exp\left(\frac{1}{\sigma^2}(\vec{x} - \vec{r})^T(\vec{x}' - \vec{x}) + \frac{\|\vec{x}' - \vec{x}\|_2^2}{2\sigma^2}\right)\p{\vec{x}' + Z = \vec{r}}.
  \end{align*}
  $M$ is $\epsilon$-local differentially private as long as $\frac{1}{\sigma^2}(\vec{x} - \vec{r})^T(\vec{x}' - \vec{x}) + \frac{\norm{\vec{x}' - \vec{x}}_2^2}{2\sigma^2} \le \epsilon$. However, the condition does not always hold since $\vec{r}$ is any element in $\RealSet^d$. Therefore, we introduce $\delta$ such that
  \begin{align*}
    \p{\vec{x} + Z = \vec{r} : \vec{r} \in \RealSet^d, \frac{1}{\sigma^2}(\vec{x} - \vec{r})^T(\vec{x}' - \vec{x}) + \frac{\|\vec{x}' - \vec{x}\|_2^2}{2\sigma^2} > \epsilon} \le \delta.
  \end{align*}
  Let $\mathcal{E}_\epsilon = \{ \vec{r} \in \RealSet^d : \frac{1}{\sigma^2}(\vec{x} - \vec{r})^T(\vec{x}' - \vec{x}) + \frac{\|\vec{x}' - \vec{x}\|_2^2}{2\sigma^2} \le \epsilon \}$. Then, we have
  \begin{align*}
    \p{M(\vec{x}) = \vec{r}} =& \p{\vec{x} + Z = \vec{r}} \\
     =& \p{\vec{x} + Z = \vec{r}, \vec{r} \in \mathcal{E}_\epsilon} + \p{\vec{x} + Z = \vec{r}, \vec{r} \notin \mathcal{E}_\epsilon} \\
     \le& e^\epsilon\p{\vec{x}' + Z = \vec{r}} + \delta.
  \end{align*}
  Since $Z$ is a Gaussian with variance $\sigma^2\mat{I}$, $-QZ \sim \mathcal{N}(\vec{0}, \sigma^2\mat{I})$ for arbitrary orthogonal matrix $Q$. Thus, for arbitrary orthogonal matrix $Q$ we have
  \begin{align*}
    & \p{\vec{x} + Z = \vec{r} : \vec{r} \in \RealSet^d, \frac{1}{\sigma^2}(\vec{x} - \vec{r})^T(\vec{x}' - \vec{x}) + \frac{\|\vec{x}' - \vec{x}\|_2^2}{2\sigma^2} > \epsilon} \\
    =& \p{Z = \vec{r} - \vec{x} : \vec{r} \in \RealSet^d, (\vec{x} - \vec{r})^T(\vec{x}' - \vec{x}) > \sigma^2\epsilon - \frac{\|\vec{x}' - \vec{x}\|_2^2}{2}} \\
    =& \p{Z = Q^T(\vec{x} - \vec{r}) : \vec{r} \in \RealSet^d, (\vec{x} - \vec{r})^T(\vec{x}' - \vec{x}) > \sigma^2\epsilon - \frac{\|\vec{x}' - \vec{x}\|_2^2}{2}} \\
    =& \p{Z^TQ^T(\vec{x}' - \vec{x}) > \sigma^2\epsilon - \frac{\|\vec{x}' - \vec{x}\|_2^2}{2}} .
  \end{align*}
  Choosing $Q$ such that its first row vector is along with $\vec{x}' - \vec{x}$ and the others are linearly independent on $\vec{x}'-\vec{x}$ yields
  \begin{align*}
   &\p{Z^TQ^T(\vec{x}' - \vec{x}) > \sigma^2\epsilon - \frac{\|\vec{x}' - \vec{x}\|_2^2}{2}}  \\
   =& \p{Z_1\norm{\vec{x}' - \vec{x}}_2 > \sigma^2\epsilon - \frac{\|\vec{x}' - \vec{x}\|_2^2}{2}} \\
   =& \p{\frac{Z_1}{\sigma} > \frac{\sigma\epsilon}{\|\vec{x}' - \vec{x}\|_2} - \frac{\|\vec{x}' - \vec{x}\|_2}{2\sigma}} ,
  \end{align*}
  where $Z_1 \sim \mathcal{N}(0, \sigma^2)$. From the proof of the Gaussian mechanism~\citep{dwork2014algorithmic}, for $\epsilon \in (0,1)$, if $\sigma \ge \sqrt{2\ln(1.25/\delta)}B/\epsilon$, we have
  \begin{align*}
   \p{\frac{Z_1}{\sigma} > \frac{\sigma\epsilon}{\|\vec{x}' - \vec{x}\|_2} - \frac{\|\vec{x}' - \vec{x}\|_2}{2\sigma}} \le \delta.
  \end{align*}
\end{proof}
The proof of \cref{cor:data-priv} is direct application of \cref{thm:local-gaussian}.

\end{document}